\documentclass{article}
\usepackage[utf8]{inputenc}
\usepackage[top=1 in,bottom=1 in, left=0.88 in, right=0.88 in]{geometry}
\usepackage{multicol}
\usepackage{hyperref}  
\usepackage{amsfonts}       
\usepackage[font=footnotesize]{subcaption}
\usepackage{amssymb,amsmath,amsthm,mathtools,dsfont}
\usepackage{algorithm,algpseudocode}
\usepackage{caption}
\usepackage{bm}
\newtheorem{theorem}{Theorem}[section]
\newtheorem{assumption}[theorem]{Assumption}
\newtheorem{lemma}[theorem]{Lemma}
\newtheorem{proposition}[theorem]{Proposition}
\newtheorem{corollary}[theorem]{Corollary}
\newtheorem{definition}[theorem]{Definition}

\usepackage[round]{natbib}

\begin{document}

\twocolumn
\title{A Mirror Descent Perspective of Smoothed Sign Descent}

\author{Shuyang Wang$^{1}$ \and Diego Klabjan$^{2}$}
\date{\normalsize $^1$Department of Engineering Sciences and Applied Mathematics,\\
Northwestern University, Evanston, Illinois, USA \\
$^2$Department of Industrial Engineering and Management Sciences, \\
Northwestern University, Evanston, Illinois, USA}

\maketitle

\begin{abstract}
  Recent work by \citet{woodworth2020kernel} shows that the optimization dynamics of gradient descent for overparameterized problems can be viewed as low\hskip0pt-\hskip0pt dimensional dual dynamics induced by a mirror map, explaining the implicit regularization phenomenon from the mirror descent perspective. However, the methodology does not apply to algorithms where update directions deviate from true gradients, e.g. ADAM. We use the mirror descent framework to study the dynamics of smoothed sign descent with a stability constant $\varepsilon$ for regression problems. We propose a mirror map that establishes equivalence to dual dynamics under some assumptions. By studying dual dynamics, we characterize the convergent solution as an approximate KKT point of minimizing a Bregman divergence style function, and show the benefit of tuning the stability constant $\varepsilon$ to reduce the KKT error.
\end{abstract}

\section{Introduction}
Mirror descent (MD) is an optimization method that extends gradient descent (GD) beyond Euclidean geometries \citep{DarzentasJohn1984PCaM}. Central to the MD framework is a mirror map that facilitates transformation between a primal space where iterates exist and a dual space where updates are performed. By defining an appropriate mirror map, MD can adapt to the geometry of the problem for efficient optimization. Since its introduction, MD has attracted considerable research interest in its regularization properties and has motivated development of efficient optimization algorithms \citep{beck2003mirror, radhakrishnan2020linear, azizan2021stochastic, gunasekar2021mirrorless, sun2022mirror, sun2023unified}.

Recent studies reveal the power of adopting an MD perspective to interpret the optimization dynamics of GD for overparameterized problems \citep{woodworth2020kernel, li2022implicit}. Given a parameterization of a problem, they formulate mirror maps that establish equivalence between GD dynamics and low-dimensional MD dynamics. The simplified dual dynamics lead to a characterization of the convergent solution among all solutions in terms of the Bregman divergence. The convergent GD solution minimizes the Bregman divergence from the starting point. This method is further used to analyze the effects of the initialization shape \citep{azulay2021implicit} and stochasticity \citep{pesme2021implicit} on the convergent solution.

Such results have been shown on data where optimal solutions are easy to find, yet the underlying optimization dynamics are non-trivial. The MD framework provides a powerful and elegant tool for analyzing high-dimensional optimization dynamics. However, the existence of such mirror maps is highly dependent on both the problem parameterization and the optimization algorithm. Existing analyses do not extend to many popular algorithms beyond (stochastic) GD. The challenges arise from both the formulation of a mirror map and the analysis of dual dynamics. For instance, for adaptive gradient methods with coordinate-wise adaptive learning rates, the update directions deviate from the true gradients. The adaptivity alters the fundamental structure of the underlying dynamics, rendering the current methodology inapplicable. Our work addresses this limitation and proposes a method of applying the MD framework to study optimization dynamics when update directions do not follow true gradients. 

Among adaptive gradient descent methods, we examine a prototypical algorithm, smoothed sign descent, which can be viewed as a smoothed version of sign descent with a stability constant $\varepsilon$. Recent work reveals a deep connection between smoothed sign descent and popular optimizers such as ADAM and RMSProp \citep{kunstner2023noise, ma2022qualitative, balles2018dissecting, bernstein2018signsgd}. While sign descent has been studied as a proxy to understand the dynamics of more complex adaptive gradient methods \citep{ma2023understanding, balles2020geometry}, studies \citep{wang2021implicit, wang2022does} show that the stability constant plays a key role in determining the convergence direction for classification problems. This underscores the importance of studying smoothed sign descent and investigating the effect of the stability constant $\varepsilon$, which has been underexplored in literature. We study the dynamics of smoothed sign descent for a quadratically parameterized regression problem. Our results highlight the distinct properties in contrast to GD dynamics, and explicitly show the relationship between the stability constant $\varepsilon$ and the convergent solution.

In this work, we present an analysis of MD to interpret the optimization dynamics of smoothed sign descent. We identify an initial warm-up stage unique to smoothed sign descent, which allows us to formulate a mirror map for the main stage of the dynamics. Using the mirror map, we project the complex primal dynamics onto the dual space with a simplified structure. We further decompose the dual dynamics into a sign descent stage and a convergence stage. The dual dynamics interpretation enables us to connect the convergent solution to the approximate KKT point of minimizing a Bregman divergence style function. An in-depth analysis of the stability constant $\varepsilon$ reveals its effect on reducing the KKT error, corroborating the empirical findings on the sensitivity of the training and testing performance to the stability constant \citep{de2018convergence,liu2019variance,choi2019empirical}.

Our contributions are as follows.
\begin{itemize}
    \item We introduce the dual dynamics of smoothed sign descent for a quadratically parameterized regression problem using the MD framework.
    \item We show that after a warm-up stage, the dual dynamics begin a sign descent stage characterized by approximately linear growth with similar rates in all coordinates, and then transition into a convergence stage characterized by diminishing magnitude of gradients.
    \item We prove that the convergent solution satisfies the approximate KKT conditions for minimizing a Bregman divergence style function, in contrast to the already known exact Bregman divergence minimization property of GD dynamics. The convergent solution found by smoothed sign descent is the one that approximately minimizes the Bregman divergence style function from the starting point.
    \item We theoretically analyze the effect of the stability constant $\varepsilon$ on bounding the deviation from the exact KKT point, emphasizing the benefit of tuning the stability constant. 
\end{itemize}
In Section 2, we review previous research on the properties of MD and smoothed sign descent. In Section 3, we present our main results, including the formulation of dual dynamics and the characterization of convergent solutions. We conclude the paper in Section 4. 

\section{Related Work}
Recent work applies the MD framework to interpret dynamics of neural network training. \citet{woodworth2020kernel} discover the equivalent low-dimensional MD dynamics for the optimization dynamics of GD for overparameterized models, focusing on the effect of initialization scale. However, extending their methodology to more general cases remains a challenge. \citet{li2022implicit} identify a commutative property of neural network parameterization that enables the formulation of equivalent MD dynamics. \citet{pesme2021implicit} use a time-varying mirror map for stochastic GD and show the benefit of stochasticity for inducing sparsity of the convergent solution. \citet{azulay2021implicit} propose a warping technique to study the effect of the initialization shape on the equivalent MD dynamics of GD. We contribute to this line of research dealing with strict gradients by extending the framework beyond GD to a case where the adaptive learning rate breaks the gradient structure and showing distinct properties of the dual dynamics.

Research on regularization properties of MD algorithms dates back to the work \citep{beck2003mirror}, which reveals a local regularization effect in terms of Bregman divergence at each iteration. Recent study \citep{gunasekar2018characterizing} shows that MD converges to the solution that minimizes the associated Bregman divergence from the starting point among all solutions. Subsequent work \citep{azizan2018stochastic, azizan2021stochastic} extends this analysis to stochastic MD for nonlinear models and prove the Bregman divergence minimization property. Research so far primarily focuses on standard MD settings, where the dynamics follow the gradient directions in the dual space. In contrast, we study the case where the dual dynamics deviate from the gradients. We show that the convergent solution of smoothed sign descent satisfies the approximate KKT condition of minimizing a Bregman divergence style function by bounding the cumulative deviation. 

The stability constant $\varepsilon$, designed to ensure numerical stability for algorithms such as ADAM and RMSProp, is typically set to a negligible value by default. Its impact on optimization dynamics is underexplored. \citet{de2018convergence} experiment with different values of $\varepsilon$ for ADAM and RMSProp and observe that training and testing performance is sensitive to $\varepsilon$. Studies \citep{nado2020evaluating,liu2019variance,choi2019empirical} also provide empirical evidence supporting the benefit of tuning the stability constant $\varepsilon$. \citet{yuan2020eadam} study the effect of modifying the location of $\varepsilon$ in ADAM and proposes an alternative optimizer for improved performance. We provide a theoretical justification for tuning the stability constant $\varepsilon$ by explicitly showing its role in reducing the KKT error of the convergent solution.

Our work also contributes an MD perspective to the ongoing discussion on the implicit regularization phenomenon in neural network training \citep{neyshabur2014search, zhang2021understanding}. While many studies \citep{soudry2018implicit, arora2019implicit, lyugradient} focus on GD, fewer have investigated adaptive gradient methods despite the performance gap observed in the paper \citep{wilson2017marginal}. Notably, studies \citep{wang2021implicit, wang2022does} find that ADAM achieves the same convergent direction as GD in classification problems, while we prove a distinct regularization property for smoothed sign descent compared to GD in regression problems. Recent study \citep{xie2024implicit} characterizes the convergent solution of AdamW as training time approaches infinity. In contrast, we characterize the entire dynamics of smoothed sign descent by formulating the equivalent dual dynamics which reveal an intrinsically simplified structure.
\section{Dual Dynamics of Smoothed Sign Descent}
\subsection{Background}
Let us consider the update rule of GD for minimizing a loss function $L(\bm\beta)$ with step size $\eta > 0$:
\begin{equation}
    \bm\beta_{t+1} = \bm\beta_t - \eta \nabla L(\bm\beta_t).\label{eq:GDupdaterule}
\end{equation}
We suppose that the iterates $\bm\beta_t$ lie in the Euclidean space $\mathbb{R}^D$. Formally, the gradients $\nabla L(\bm\beta_t)$ lie in the dual space $\mathbb{R}^D$. In GD, we obtain the updated point by directly taking a linear combination of the iterate and the gradient as in \eqref{eq:GDupdaterule}. MD, however, formally distinguishes the primal and the dual spaces using a mirror map to transform between them. A mirror map $\nabla \Phi: \mathbb{R}^D \to \mathbb{R}^D$ is defined as the gradient of a potential function $\Phi: \mathbb{R}^D \to \mathbb{R}$, which is any differentiable and strictly convex function. The mirror map $\nabla \Phi$ maps the primal variable $\bm\beta$ to the dual variable denoted by $\bm\phi \in \mathbb{R}^D$. Each iteration of MD for minimizing $L(\bm\beta)$ follows the following steps, where the step size $\eta > 0$:
\begin{align}
    \bm\phi_t &= \nabla \Phi (\bm\beta_t) \label{eq:dualvariable} \\
    \bm\phi_{t+1} &= \bm\phi_t - \eta \nabla L(\bm\beta_t) \label{eq:discreteMDupdate}\\
    \bm\beta_{t+1} &= (\nabla \Phi)^{-1}(\bm\phi_{t+1}). \label{eq:backtoprimalvariable}
\end{align}
By plugging in \eqref{eq:dualvariable}, we can rewrite the MD update \eqref{eq:discreteMDupdate} in the dual space as: 
\begin{equation}
    \nabla \Phi(\bm\beta_{t+1}) = \nabla \Phi(\bm\beta_t) - \eta \nabla L(\bm\beta_t)).
\end{equation}
In the continuous-time limit when $\eta \to 0$, we get the \textbf{dual dynamics} of $\bm \beta(t)$:
\begin{equation}
    \frac{d\nabla \Phi(\bm\beta(t))}{dt} = -\nabla L(\bm\beta(t)). \label{eq:MDdualdynamics}
\end{equation}

A key element of MD is the Bregman divergence that serves as the notion of measuring the distance between two points in the primal space.
\begin{definition}[Bregman divergence]
For $\bm\beta_1, \bm\beta_2 \in \mathbb{R}^D$, the Bregman divergence associated with a potential function $\Phi$ from $\bm\beta_1$ to $\bm\beta_2$ is defined as
\begin{equation*}
    D_{\Phi}(\bm\beta_1, \bm\beta_2) = \Phi(\bm\beta_1) - \Phi(\bm\beta_2) - \langle \bm\beta_1 - \bm\beta_2, \nabla \Phi(\bm\beta_2)\rangle.
\end{equation*}
\end{definition} Bregman divergence generalizes squared Euclidean distance and captures different geometric structure of the space through the choice of $\Phi$. When $\Phi(\bm\beta) = \frac{1}{2}\Vert \bm\beta \Vert_2^2$, the associated Bregman divergence reduces to the squared Euclidean distance, the mirror map $\nabla \Phi$ becomes an identity map, and MD simplifies to GD.

\subsection{Problem Setup}
We suppose that there are $N$ examples with $D > N$ features $\{(\bm{x}^{(i)}, y^{(i)})\}_{i=1, ..., N}$, where $\bm{x}^{(i)} \in \mathbb{R}^D, y^{(i)} \in \mathbb{R}$. Let us denote the data matrix by $X \in \mathbb{R}^{N \times D}$, where each row is $\bm{x}^{(i)}$, and denote the labels of the examples by $\bm{y} \in \mathbb{R}^N$. The Hadamard product is denoted by $\odot$. We consider a regression problem of minimizing the following loss function with respect to $\bm{w} := \begin{bmatrix}
    \bm{w}^+ \\ \bm{w}^-
\end{bmatrix} \in \mathbb{R}^{2D}$, where $\bm{w}^+, \bm{w}^- \in \mathbb{R}^D$:
\begin{align}\label{eq:regressionproblem}
    L(\bm{w}) = \frac{1}{4}&\left(X\left(\bm{w}^+ \odot \bm{w}^+ - \bm{w}^- \odot \bm{w}^-\right) - \bm{y}\right)^\top \nonumber \\
    &\left(X\left(\bm{w}^+ \odot \bm{w}^+ - \bm{w}^- \odot \bm{w}^-\right) - \bm{y}\right).
\end{align}
We let $\bm{\beta} := \bm{w}^+ \odot \bm{w}^+ - \bm{w}^- \odot \bm{w}^- \in \mathbb{R}^D$ denote the regression parameter, and $L(\bm\beta) = \frac{1}{4}\left(X\bm\beta - \bm y\right)^\top\left(X\bm\beta - \bm y\right)$ is the standard quadratic loss. This parameterization of $\bm\beta$ by $\bm w$ can also be viewed as a 2-layer diagonal linear neural network with weights $\bm{w} \in \mathbb{R}^{2D}$ (see Section 4 of the paper \citep{woodworth2020kernel} for a detailed study of the model). Despite its simplicity, this setup has been used to prove insightful results for neural networks training \citep{woodworth2020kernel,pesme2021implicit,nacson2022implicit,vivien2022label}. 

When GD is applied to minimize loss \eqref{eq:regressionproblem} with respect to $\bm w$, from the GD update rule with infinitesimal step size $\eta$ we get\begin{equation*}
    \frac{d\bm w^+(t)}{dt} = -\nabla_{\bm w^+}L(\bm w(t)),~\frac{\bm w^-(t)}{dt} = -\nabla_{\bm w^-}L(\bm w(t)).
\end{equation*} Then, by chain rule we get the optimization dynamics of $\bm\beta(t)$: \begin{align}
    \frac{d\bm\beta(t)}{dt} = &-2\bm w^+(t) \odot \nabla_{\bm w^+} L(\bm w(t))\nonumber \\
    &+ 2\bm w^-(t) \odot \nabla_{\bm w^-} L(\bm w(t)).
    \label{eq:GDprimaldynamics}
\end{align}
Previous work \citep{woodworth2020kernel} shows that by defining a potential function:\begin{equation}
    \Psi_{\alpha}(\bm\beta) := \frac{1}{4}\left (\sum_{i=1}^D \beta_i \operatorname{arcsinh}\left(\frac{\beta_i}{2\alpha^2}\right) - \sqrt{\beta_i^2 + 4\alpha^4}\right), \label{eq:GDpotential}
\end{equation}
where $\alpha > 0$ is the initialization scale, we can project the dynamics \eqref{eq:GDprimaldynamics} onto the dual space using the mirror map $\nabla\Psi_{\alpha}$. Here the gradient is taken with respect to $\bm\beta$. Then, by derivation in Appendix~\ref{appendixC}, it follows that the dual dynamics are given by: 
\begin{equation}
    \frac{d\nabla\Psi_{\alpha}(\bm\beta(t))}{dt} = -\nabla_{\bm\beta} L(\bm\beta(t)).\label{eq:GDdualdynamics}
\end{equation}
Since \eqref{eq:GDprimaldynamics} and \eqref{eq:GDdualdynamics} are equivalent, in the continuous-time limit, the evolution of $\bm\beta(t)$ using GD can be interpreted as following the MD algorithm \eqref{eq:dualvariable}-\eqref{eq:backtoprimalvariable} with mirror map $\nabla\Psi_{\alpha}$.

The dual dynamics \eqref{eq:GDdualdynamics} reveal an intrinsically low-dimensional structure of the dynamics of $\bm\beta(t)$ in the overparameterized setting where $N < D$. Specifically, the gradients $\nabla_{\bm\beta}L(\bm\beta)$ in the right hand side of \eqref{eq:GDdualdynamics} are confined in a subspace $\text{span}\{\bm{x}^{(1)}, ..., \bm{x}^{(N)}\}$, which has dimension of at most $N$. Furthermore, by analyzing the dual dynamics, previous work \citep{woodworth2020kernel} proves that the convergent solution $\bm{\beta}^{\infty} := \lim_{t\to\infty}\bm\beta(t)$ satisfies the KKT conditions of the constrained optimization problem:

\begin{equation}\label{eq:GDconvergentsolution}
    \bm\beta^{\infty} = \underset{\bm\beta \in \mathbb{R}^D \text{ s.t. } X\bm\beta = \bm{y}}{\operatorname{argmin}}\,D_{\Psi_\alpha}(\bm\beta, \bm\beta(0)).
\end{equation}
\sloppy
In this work, we study the dynamics of smoothed sign descent for minimizing \eqref{eq:regressionproblem}. For smoothed sign descent, the weights are updated according to 
\begin{equation*}
    \bm{w}_{t+1} = \bm{w}_t - \eta \cdot \frac{\nabla_{\bm{w}} L(\bm{w}_t)}{|\nabla_{\bm{w}} L(\bm{w}_t)| + \varepsilon \bm{1}},
\end{equation*}
where $\varepsilon > 0$ is the stability constant and the operations are taken element-wise. Smoothed sign descent can be viewed as an adaptive gradient method with coordinate-wise adaptive learning rate $\eta_{i, t} = \frac{\eta}{|\left[\nabla_{\bm{w}}L(\bm{w}_t) \right]_i| + \varepsilon}$ for each $i$.  We suppose that the weights are initialized by $\bm{w}(0) = \alpha \mathbf{1}$, $\alpha > 0$. In the continuous-time limit, the dynamics of the weights become
\begin{equation}
    \frac{d\bm{w}(t)}{dt} = -\frac{\nabla_{\bm{w}}L(\bm{w}(t))}{|\nabla_{\bm{w}}L(\bm{w}(t))| + \varepsilon\mathbf{1}} \label{eq:weightdynmics}.
\end{equation}
This yields the dynamics of the regression parameter $\bm\beta(t)$ as follows, with $\bm\beta(0) = \bm{0}$:
\begin{align}
    \frac{d\bm{\beta}(t)}{dt} = &-2\bm w^+(t) \odot \frac{\nabla_{\bm{w}^+}L(\bm{w}(t))}{|\nabla_{\bm{w}^+}L(\bm{w}(t))| + \varepsilon\mathbf{1}} \nonumber \\
    &+ 2\bm w^-(t) \odot \frac{\nabla_{\bm{w}^-}L(\bm{w}(t))}{|\nabla_{\bm{w}^-}L(\bm{w}(t))| + \varepsilon\mathbf{1}}. \label{eq:betadynmics}
\end{align}

With coordinate-wise adaptive learning rate, the update direction deviates from the true gradients and the mirror map $\nabla \Psi_{\alpha}$ for GD no longer holds. It leads to two interesting questions: \begin{enumerate}
    \item Can we formulate a mirror map to show equivalent dual dynamics for \eqref{eq:betadynmics}? 
    \item Can we use the dual dynamics to characterize the convergent solution among all solutions?
\end{enumerate}

\subsection{Main Results}
In this section, we present our answers to the two questions. We construct a mirror map for smoothed sign descent that reveals a simplified structure of the dual dynamics. We analyze different stages of the induced dual dynamics, and prove that the convergent solution satisfies approximate KKT conditions for minimizing a Bregman divergence style function, which is also defined in \citep{pesme2021implicit}. The weight dynamics \eqref{eq:weightdynmics} form a coupled system of nonlinear ODEs, with the stability constant $\varepsilon$ adding another layer of complexity. Solving this ODE system analytically is intractable. Therefore, we make the following assumptions to facilitate our analysis.
\begin{assumption}\label{assumption:data}
    We assume that $y^{(n)}$ are non-zero, and that there exists a permutation of the columns of $X$ such that $X^\top X$ is block-diagonal with $N$ rank-1 blocks denoted by $B^{(n)} \in \mathbb{R}^{D_n \times D_n}$ for $n = 1, \dots, N$.
\end{assumption}

It is easy to see that this condition is equivalent to requiring that each row of $X$ has $D_n \geq 1$ non-zero elements denoted by $x^{(n)}_1, \dots, x^{(n)}_{D_n}$, where $\sum_{n=1}^N D_n = D$. While this assumption yields an easy optimization problem in the primal space, the dynamics of smoothed sign descent are very complex and intriguing. We require the stability constant $\varepsilon$ to be small relative to components of the initial gradient so that it does not overshadow the essential behavior of the dynamics as a smoothed version of sign descent. We notice that $\bm{w} = \bm{0}$ is a stationary point of the weight dynamics \eqref{eq:weightdynmics}. Since the weights are initialized as $\bm{w}(0) = \alpha \bm{1}$ where $\alpha > 0$, we assume that $\alpha$ is chosen not so small to avoid being stuck near a stationary point, and also not so large that it dominates the final convergent solution. 

\begin{assumption}\label{assumption:epsilonalpha}
\sloppy
We assume that for each $n \in \{1, \dots, N\}$ and $i \in \{1, \dots, D_n\}$, the stability constant $\varepsilon$ and the initialization scale $\alpha$ satisfy: 
\begin{align*}
    &0 \leq \varepsilon \leq \frac{1}{9}\frac{|x^{(n)}_i| |y^{(n)}|^{\frac{3}{2}}}{\sqrt{2\sum_{k=1}^{D_n} |x_k^{(n)}|}}, \\
    &\frac{9\varepsilon}{4\left|x^{(n)}_i y^{(n)}\right|} \leq \alpha \leq \frac{1}{3}\sqrt{\frac{|y^{(n)}|}{2\sum_{k=1}^{D_n} |x_k^{(n)}|}}.
\end{align*}
\end{assumption}

\subsubsection{Three Stages}\label{section:threestages}
We begin by studying the sign and monotonicity of $\bm w^+(t)$ and $\bm w^-(t)$ by the following lemma assuming they satisfy \eqref{eq:weightdynmics}. Proofs of the results in this section can be found in Appendix~\ref{appendixA}.
\begin{proposition}\label{proposition:signandmonotone}
For each coordinate $i \in \{1, \dots, D\}$, 
\begin{itemize}
    \item $w_i^+(t)$ and $w_i^-(t)$ are always non-negative,
    \item if $w_i^+(0)' > 0$, then $w^+_i(t)' \geq 0$ and $w^-_i(t)' \leq 0$ for all $t$,
    \item if $w_i^+(0)' \leq 0$, then $w^+_i(t)' \leq 0$ and $w^-_i(t)' \geq 0$ for all $t$.
\end{itemize}
\end{proposition}
For each $i$, based on this proposition, either $w^+_i(t)$ or $w^-_i(t)$ is monotonically non-decreasing. We denote the dominating weight that is monotonically non-decreasing by $u_i$, and we denote the one that is non-increasing by $v_i$. 
A key identity in the derivation of the mirror map for GD is that $w^+_i(t)w^-_i(t) = \alpha^2$ holds throughout the dynamics. However, this quantity is not conserved when coordinate-wise adaptivity is applied. In fact, we can show that $w^+_i(t) w^-_i(t) < \alpha^2$ for $t > 0$. The adaptive learning rate ensures similar rate of change across all coordinates, and enables sufficient updates even when the gradient magnitude is relatively small. In particular, this allows the non-dominating weight $v_i(t)$ to diminish to negligible values early on. Based on this observation, we identify an initial warm-up stage of the dynamics where $v_i(t)$ decreases to and remains below a value on the order $\varepsilon$ across all coordinates. The following lemma also shows that this warm-up stage lasts no longer than $t=2\alpha$.
\begin{proposition}\label{proposition:initialstage}
There exists $T_0 \in (0, 2\alpha]$ such that for all $t \geq T_0$, $v_i(t) \leq \frac{2\varepsilon}{|x^{(n)}_i y^{(n)}|}$ for all $i$. 
\end{proposition}
\sloppy
The proof hinges on upper bounding the value of $v_i(t)$ when the gradient component $[\nabla_{\bm v}L(\bm w(t))]_i$ reaches $\varepsilon$ at $t=t_i$. Before $t_i$, the absolute value of the derivative $|v'_i(t)|$ is always greater than $\frac{1}{2}$, ensuring rapid decreasing of $v_i(t)$. Meanwhile, the non-negativity of $v_i(t)$ by Proposition~\ref{proposition:signandmonotone} guarantees that the rapid decreasing stage lasts no longer than $2\alpha$. Based on the expression $[\nabla_{\bm v}L(\bm w(t))]_i = v_i(t)|x^{(n)}_i r^{(n)}(t)|$, we continue to lower bound the residual $|r^{(n)}(t)|$ using the maximal growth of $u_i(t)$ during this short time period. Finally, the lower bound of $|r^{(n)}(t_i)|$ leads to the upper bound of $v_i(t_i)$ at $t_i$. We complete the proof by letting $T_0$ be the largest $t_i$ across all coordinate $i$. 

During the warm-up stage, both $\bm{u}(t)$ and $\bm{v}(t)$ follow sign descent approximately, which allows us to approximate the primal dynamics of $\bm\beta(t)$ by sign descent. After $T_0$, the dynamics of $\bm\beta(t)$ transition into the main stage, where $\bm v(t)$ remains small and the magnitude of $\bm \beta(t)$ is denominated by $\bm u(t)$. While the primal dynamics become complex, we formulate a mirror map so that the dual dynamics have a simplified structure that closely aligns with the sign of $\nabla_{\bm u} L(\bm w(t))$.
\begin{proposition}[Dual dynamics of smoothed sign descent]\label{proposition:dualdynamics}
    For $t > 0$, we define a potential function $\Phi_t(\bm\beta) = \frac{2}{3}\sum_{i=1}^{D} \left(|\beta_i| + v_{i, t}^2\right)^{\frac{3}{2}}$. The induced mirror map $\nabla \Phi_t: \mathbb{R}^D \to \mathbb{R}^D$ maps $\bm\beta(t)$ to the dual space. The dynamics in the dual space follow \begin{equation}\label{eq:dualdynamics}
        \frac{d\nabla \Phi_t (\bm{\beta}(t))}{dt} = -\operatorname{sgn}(\bm\beta(t)) \odot \frac{\nabla_{\bm{u}}L(\bm w(t))}{|\nabla_{\bm{u}}L(\bm w(t))| + \varepsilon\bm{1}}.
    \end{equation}
\end{proposition}
The potential function is time-varying with a time-dependent parameter $v_{i, t} := v_i(t)$. \citet{pesme2021implicit} also employ a time-varying potential function to construct a mirror map for the dynamics of stochastic GD. \citet{radhakrishnan2020linear} conduct a thorough analysis of the convergence of MD with time-dependent mirrors. For $t \geq T_0$, since the non-dominating weights $v_i(t)$ diminish to small values by Proposition~\ref{proposition:initialstage}, the potential function has a close connection with the $l_{3/2}$-norm of $\bm\beta(t)$, in contrast with the potential function \eqref{eq:GDpotential} for GD.

The dual dynamics \eqref{eq:dualdynamics} indeed reveal a greatly simplified structure compared to the primal dynamics \eqref{eq:betadynmics}. However, it differs from standard MD dynamics \eqref{eq:MDdualdynamics} where the updates in the dual space align with the gradients exactly. The alignment has allowed previous work to show that the convergent solution satisfies the KKT conditions for Bregman divergence minimization as in \eqref{eq:GDconvergentsolution}. Therefore, further analysis of the dual dynamic \eqref{eq:dualdynamics} is required to understand the deviation from following the true gradients. 
\begin{proposition}\label{proposition:mainstage}
    There exists $T > T_0$ such that we can divide the dynamics into two stages: 
    \begin{itemize}
        \item Sign descent stage: for $t \in [T_0, T)$, $\left|\nabla_{\bm u} L(\bm w(t))\right|_i > \varepsilon$ for all $i$,
        \item Convergence stage: for $t \in [T, \infty)$, $\min_i \left|\nabla_{\bm u} L(\bm w(t))\right|_i \leq \varepsilon$.
    \end{itemize}
\end{proposition}
At the beginning, the dual dynamics resemble sign descent when gradient components are relatively large compared to $\varepsilon$. The stability constant comes into effect when $\left|\nabla_{\bm u} L(\bm w )\right|_i$ becomes small. In Proposition~\ref{proposition:mainstage}, we prove the transition between the two stages by studying the evolution of the magnitude of each gradient component. Importantly, Proposition~\ref{proposition:mainstage} shows that once a gradient value reaches $\varepsilon$, it remains small for the duration of the dynamics. The dynamics then enter a convergence stage with diminishing magnitude of gradients. Eventually, the dynamics approximate the direction of $\nabla_{\bm u} L(\bm w)$ as all gradient components approach zero (see Lemma~\ref{lemma:convergence} in Appendix~\ref{appendixA}). 

We illustrate the transition of the three stages in Figure~\ref{fig:threestages}. We randomly generate a dataset with $N=2$ and $D=5$ that satisfies Assumption~\ref{assumption:data} and set the initialization scale $\alpha=0.1$. We simulate the dynamics \eqref{eq:weightdynmics} of smoothed sign descent using the ODE solver in \texttt{SciPy} and visualize the evolution of primal and dual variables. In the experiments, $T_0$ is calculated as the value when $\max_i |\nabla_{\bm v}L(\bm w(t))|_i$ first becomes $\varepsilon$, while $T$ is calculated as the value when $\min_i |\nabla_{\bm u}L(\bm w(t))|_i$ first becomes $\varepsilon$. Based on smoothed sign descent (see \eqref{eq:dualdynamics}) and Proposition~\ref{proposition:mainstage}, we expect the change to be linear in $[T_0, T]$, and incoherent behavior in $[T, \infty)$.
In the initial stage when $t  < T_0$, we observe that the primal variable has linear change across all coordinates. During the sign descent stage when $T_0 \leq t < T$, the dual variable continues growing linearly with approximately uniform rate in all coordinates, while $\bm\beta(t)$ no longer changes linearly. After $T$, the dynamics enter the convergence stage, where the primal and dual variables gradually approach the convergent point. We also observe that the value of $\varepsilon$ plays a key role in shaping the dynamics. For smaller $\varepsilon$, the dual variable follows the sign descent more closely and converges to values concentrated around two distinct points across all coordinates; while for larger $\varepsilon$, the dual variable shows greater dispersion among the values across all coordinates at convergence. We quantify the relationship between the value of $\varepsilon$ and the convergent solution in the following analysis.

\begin{figure*}[ht]
    \centering
    \begin{subfigure}[b]{0.4\textwidth}
    \centering
    \includegraphics[width=\textwidth]{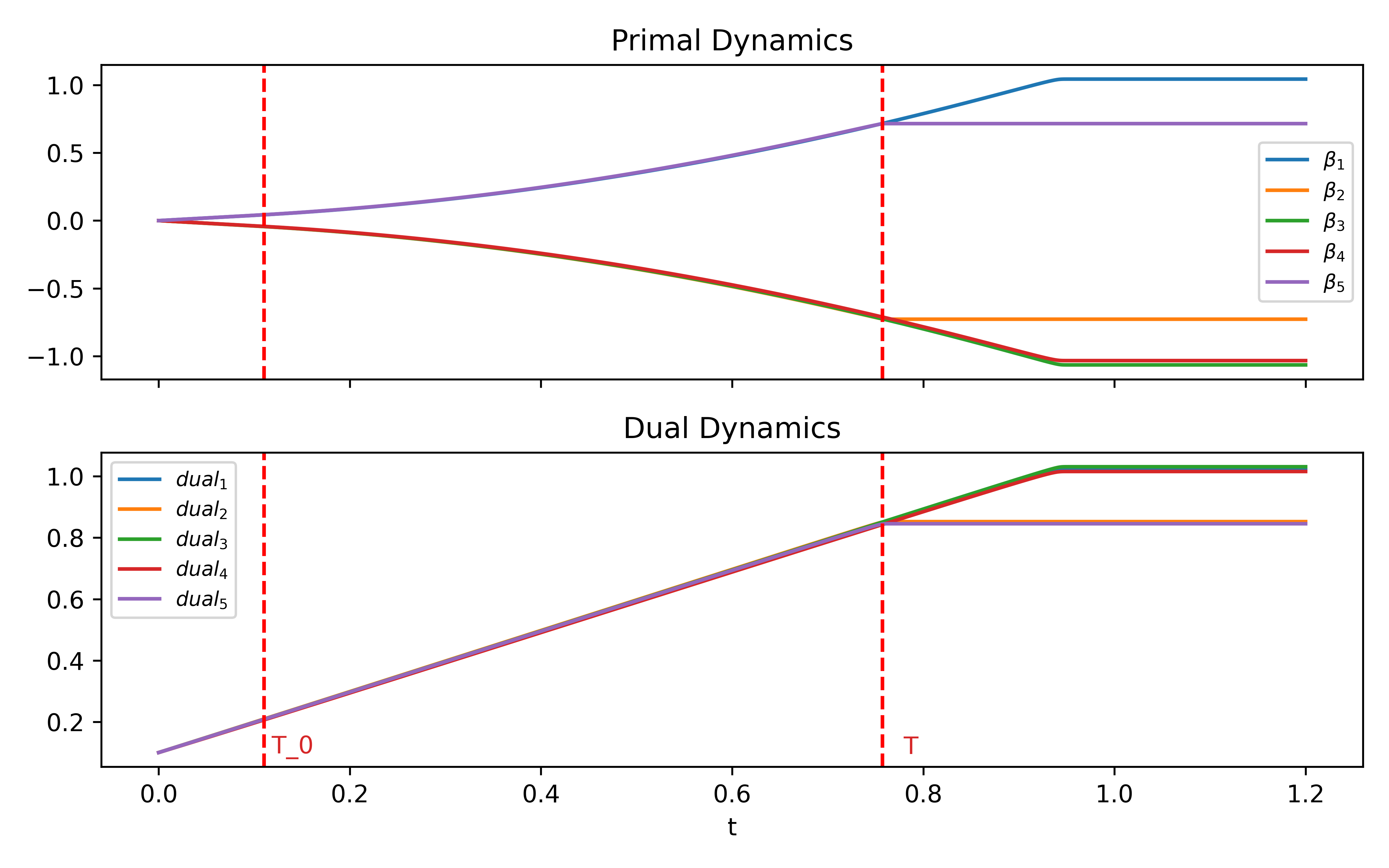}
    \caption{$\varepsilon = 0.001$}
    \end{subfigure}
    \begin{subfigure}[b]{0.4\textwidth}
    \centering
    \includegraphics[width=\textwidth]{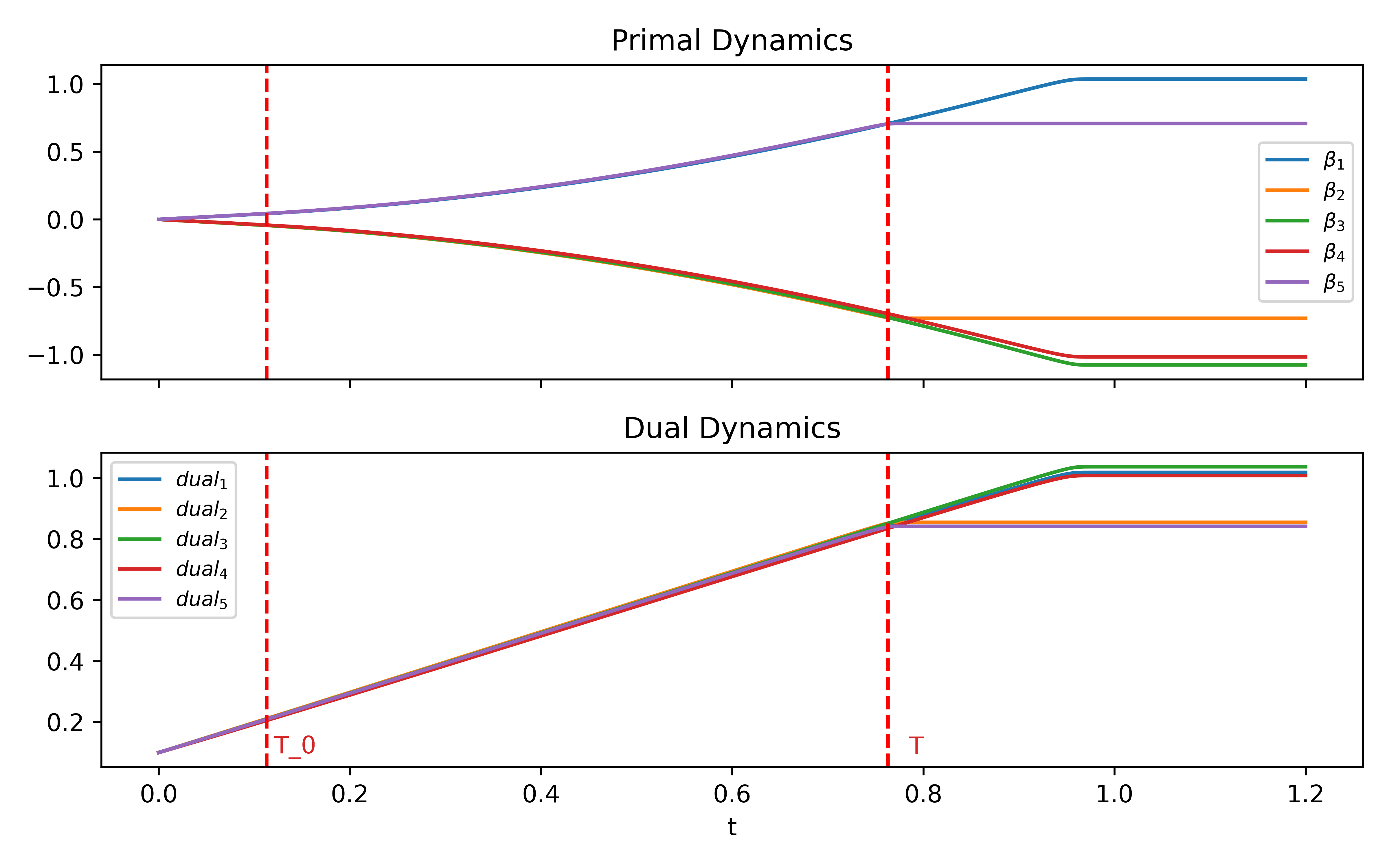}
    \caption{$\varepsilon = 0.002$}
    \end{subfigure}
    \vfill
     \begin{subfigure}[b]{0.4\textwidth}
    \centering
    \includegraphics[width=\textwidth]{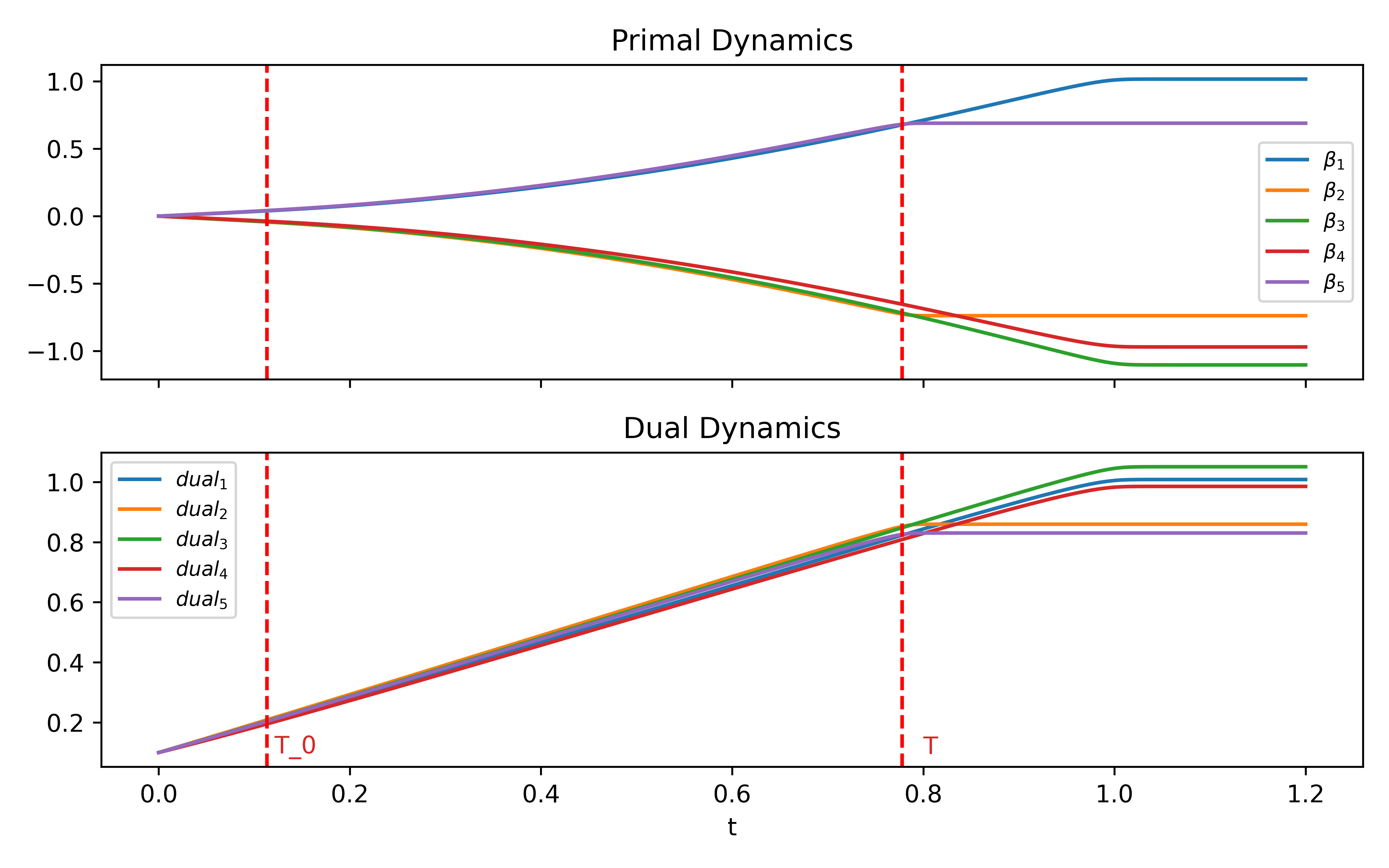}
    \caption{$\varepsilon = 0.005$}
    \end{subfigure}
    \begin{subfigure}[b]{0.4\textwidth}
    \centering
    \includegraphics[width=\textwidth]{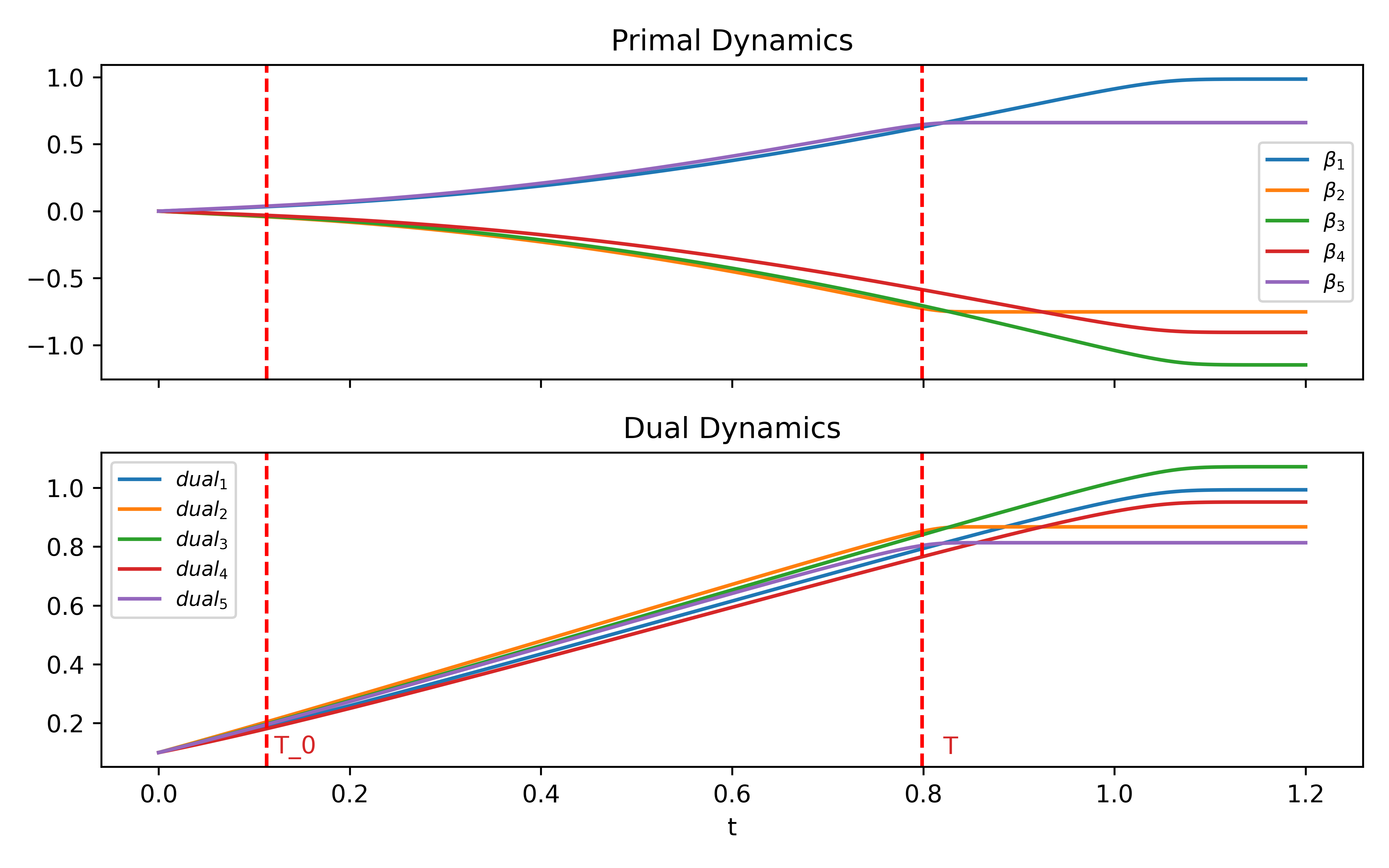}
    \caption{$\varepsilon = 0.01$}
    \end{subfigure}
    \caption{Evolution of primal variable $\bm\beta(t)$ and dual variable $\nabla \Phi_t (\bm\beta(t))$ in $\mathbb{R}^5$ of smoothed sign descent with different values of stability constant $\varepsilon$. The vertical line $t=T_0$ marks the transition from warm-up stage to the sign descent stage, and the line $t=T$ marks the transition to the convergence stage.}
    \label{fig:threestages}
\end{figure*}

\subsubsection{Characterization of Convergent Solution by Bregman Divergence}\label{section:characterization}
The convergent solution of smoothed descent dynamics deviates from the exact KKT point of Bregman divergence minimization. However, we can show that it satisfies the $\delta$-KKT conditions for a Bregman divergence style function. In this section, we build on the results about stage transitions and conduct an in-depth analysis to quantify and bound the error $\delta$. To emphasize the role of $\varepsilon$ in bounding the error, we impose an additional assumption on the block-diagonal structure from Assumption~\ref{assumption:data}.
\begin{assumption}\label{assumption:data2}
    We assume that each block $B^{(n)}$ of the block-diagonal matrix $X^\top X$ has size $D_n=2$.
\end{assumption}
The 2D block structure enables us to derive an explicit dependence of the bounds for $\delta$ on the stability constant $\varepsilon$, while keeping the overparameterization setting for smoothed sign descent. By the spectral theorem, we can write $X^\top X = Q \Lambda Q^\top$ for an orthogonal matrix $Q$ and a diagonal matrix $\Lambda$. The matrix $Q$ is block-diagonal, where each block is expressed as a 2D rotation matrix parameterized by $\theta_n$. Then, we have
\begin{align*}
    B^{(n)} = \begin{bmatrix}
        \cos\theta_n & -\sin\theta_n \\
        \sin\theta_n & \cos\theta_n 
    \end{bmatrix} \begin{bmatrix}
        \lambda_n & 0 \\
        0 & 0 
    \end{bmatrix}
    \begin{bmatrix}
        \cos\theta_n & \sin\theta_n \\
        -\sin\theta_n & \cos\theta_n 
    \end{bmatrix},
\end{align*}
where $\lambda_n > 0$ and $\cos\theta_n, \sin\theta_n$ are non-zero by Assumption~\ref{assumption:data}. Without loss of generality, we assume that $|\cos\theta_n| \geq |\sin\theta_n|$, which can be achieved by ordering the columns of $X$. We first present the result for $N=1$ to illustrate the key findings and then generalize the results to $N > 1$. To this end, we show in Appendix~\ref{appendixA} that there exists $\bm v^{\infty} = \lim_{t\to \infty}\bm v(t)$ and we let $\Phi_{\infty}(\bm\beta) = \frac{2}{3}\sum_{i=1}^D\left(|\beta_i| + (v^\infty_{i})^2\right)^{\frac{3}{2}}$.
\begin{theorem}\label{maintheorem}
        As $t \to \infty$, the regression parameter converges to an interpolating solution. We let $\bm\beta^{\infty} := \lim_{t \to \infty}{\bm\beta(t)}$, which exists by Lemma~\ref{lemma:convergence} in Appendix~\ref{appendixA}, and let $\bm\beta_0 := \bm\beta(0)$. We define a Bregman divergence style function $E$ associated with the potential function $\Phi$ for smoothed sign descent by
        \begin{equation*}
            E(\bm\beta, \bar{\bm\beta}) := \Phi_{\infty}(\bm\beta) - \Phi_0(\bar{\bm\beta}) + \langle \nabla\Phi_{0}(\bar{\bm\beta}), \bar{\bm\beta} - \bm\beta \rangle.
        \end{equation*}
        The convergent solution $\bm\beta^\infty$ satisfies the $\delta$-KKT conditions for the constrained optimization problem: \begin{equation}
            \underset{{\bm\beta \in \mathbb{R}^D \text{ s.t. } X\bm\beta = \bm y}}{\min} E(\bm\beta, \bm\beta_0)\label{eq:Bregmanminimization}
        \end{equation}
        with the error $\delta(\varepsilon)$ bounded by $\max\left\{|M_+|, |M_-|\right\}$, where
        \begin{align*}
            M_+ := &\left( |\cos\theta_1| - |\sin\theta_1| \right) \left(\lambda_1^{-\frac{1}{4}}|y^{(1)}|^\frac{1}{2} - \frac{\sqrt{2}\varepsilon}{4\lambda_1^\frac{1}{2}|y^{(1)}|}\right), \\
            M_- := &\left(|\cos\theta_1|-|\sin\theta_1|\right)\left(\left(2\lambda_1\right)^{-\frac{1}{4}}|y^{(1)}|^{\frac{1}{2}}-\alpha\right) \\ &-2\sqrt{\frac{2\varepsilon}{\lambda_1^{\frac{3}{4}}|\sin\theta_1||y^{(1)}|^{\frac{1}{2}}}} \\
            &- \frac{3\sqrt{2}\varepsilon}{\lambda_1^{\frac{1}{2}}|\sin\theta_1 y^{(1)}|}\ln \left(\frac{\lambda_1^{\frac{1}{4}}|\sin\theta_1||y^{(1)}|^{\frac{3}{2}}}{\sqrt{2}\varepsilon}\right).
        \end{align*}
\end{theorem}

We present the main idea of the proof here and provide the full proof in Appendix~\ref{appendixB}. First, we observe the connection between the gradient of $E$ and the integral of the dual dynamics \eqref{eq:dualdynamics} with respect to $t$. 
The dual dynamics structure enables us to calculate the deviation $\delta$ from satisfying the stationary condition using the dominating weights $\bm u^\infty$. Next, using an orthogonal projection, we reduce the problem to bounding the absolute value of $\Delta := |\cos\theta_1|\left(u_2^{\infty}-u_2(0)\right) - |\sin\theta_1|\left(u_1^{\infty}-u_1(0)\right)$. Then, to bound $\Delta$, we leverage the ratios between $u'_1(t)$ and $u'_2(t)$ in different stages of the dual dynamics, and focus on bounding the key quantity $u_2(T)$ at the transition between the two stages. During the sign descent stage, the leading terms of $u'_1(t)$ and $u'_2(t)$ are both $1$ in the Taylor expansion at $\varepsilon=0$, which guarantees a lower bound for $u_2(T)$. Being in the convergence stage, $u_1(t)$ dominates the growth, which allows us to derive an upper bound for $u_2^\infty$. Finally, a lower bound for $u_2^{\infty}$ leads to $\Delta \geq M_-$, while an upper bound leads to $\Delta \leq M_+$. 

The derivation relies on the key quantity of $u_2(T)$ at the stage transition when the smallest gradient component reaches $\varepsilon$. The value of $\varepsilon$ is crucial in determining the stage transition and it eventually affects the convergent solution. We further reveal the relationship between $\varepsilon$ and the upper bound of $\delta$ in the following corollary. We provide the proof in Appendix~\ref{appendixB}.
\begin{corollary}\label{corollary:epsilon}
    We let $\mathcal{I}_\varepsilon$ be the range of $\varepsilon$ implied by Assumption~\ref{assumption:epsilonalpha}. Then, there exists a non-degenerate interval $\mathcal{I}' \subseteq \mathcal{I}_\varepsilon$ such that for all $\varepsilon \in \mathcal{I}'$, \begin{equation*}
        \delta(\varepsilon) \leq \bar{M} - \left( |\cos\theta_1| - |\sin\theta_1| \right)\frac{\sqrt{2}\varepsilon}{4\lambda_1^\frac{1}{2}|y^{(1)}|},
    \end{equation*}
    where $\bar{M} := \left( |\cos\theta_1| - |\sin\theta_1| \right)\lambda_1^{-\frac{1}{4}}|y^{(1)}|^\frac{1}{2}$ is a quantity independent of $\varepsilon$.
\end{corollary}
The result highlights the role of $\varepsilon$ in bounding the KKT error. Given a fixed dataset, choosing a larger $\varepsilon$ within a certain interval effectively shrinks the upper bound on the KKT error $\delta$. It suggests that by using a proper value of $\varepsilon$, the dynamics can converge to a solution closer to the point with the $E$ minimization property. Therefore, our result provides a theoretical ground for the benefit of tuning $\varepsilon$ versus using a small default value for adaptive gradient methods.

To visualize the convergent solutions for different values of $\varepsilon$, we plot the trajectory of $\bm\beta(t)$ using randomly generated data with $N=1$ and $D=2$ in Figure~\ref{fig:comparison-1}. We note that as $\varepsilon$ becomes larger, the convergent solution is closer to the solution with the minimal value of $E$ to the initial point among all solutions. We also compute the value of $E$ to the initial point for convergent solutions using different $\varepsilon$ and plot the trend in Figure~\ref{fig:comparison-2}. The plot confirms that for larger $\varepsilon$, the convergent solutions have smaller values of $E(\bm\beta^\infty, \bm\beta_0)$.

\begin{figure}[ht]
    \centering
    \includegraphics[width=0.8\linewidth]{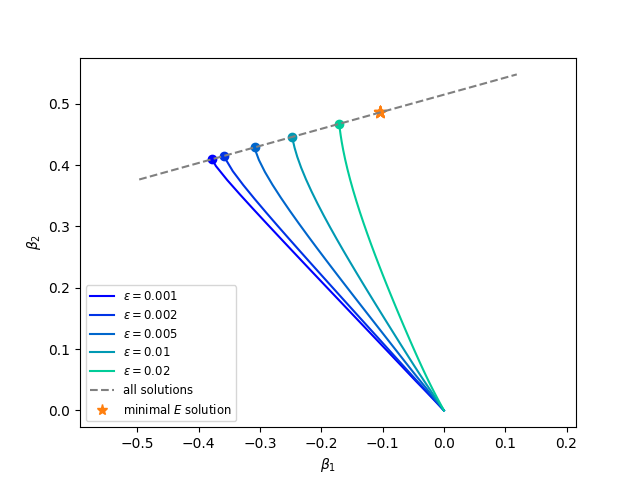}
    \caption{Trajectories of $\bm\beta(t)$ in $\mathbb{R}^2$ for different values of stability constant $\varepsilon$.}
    \label{fig:comparison-1}
\end{figure}

\begin{figure}[ht]
    \centering
    \includegraphics[width=0.75\linewidth]{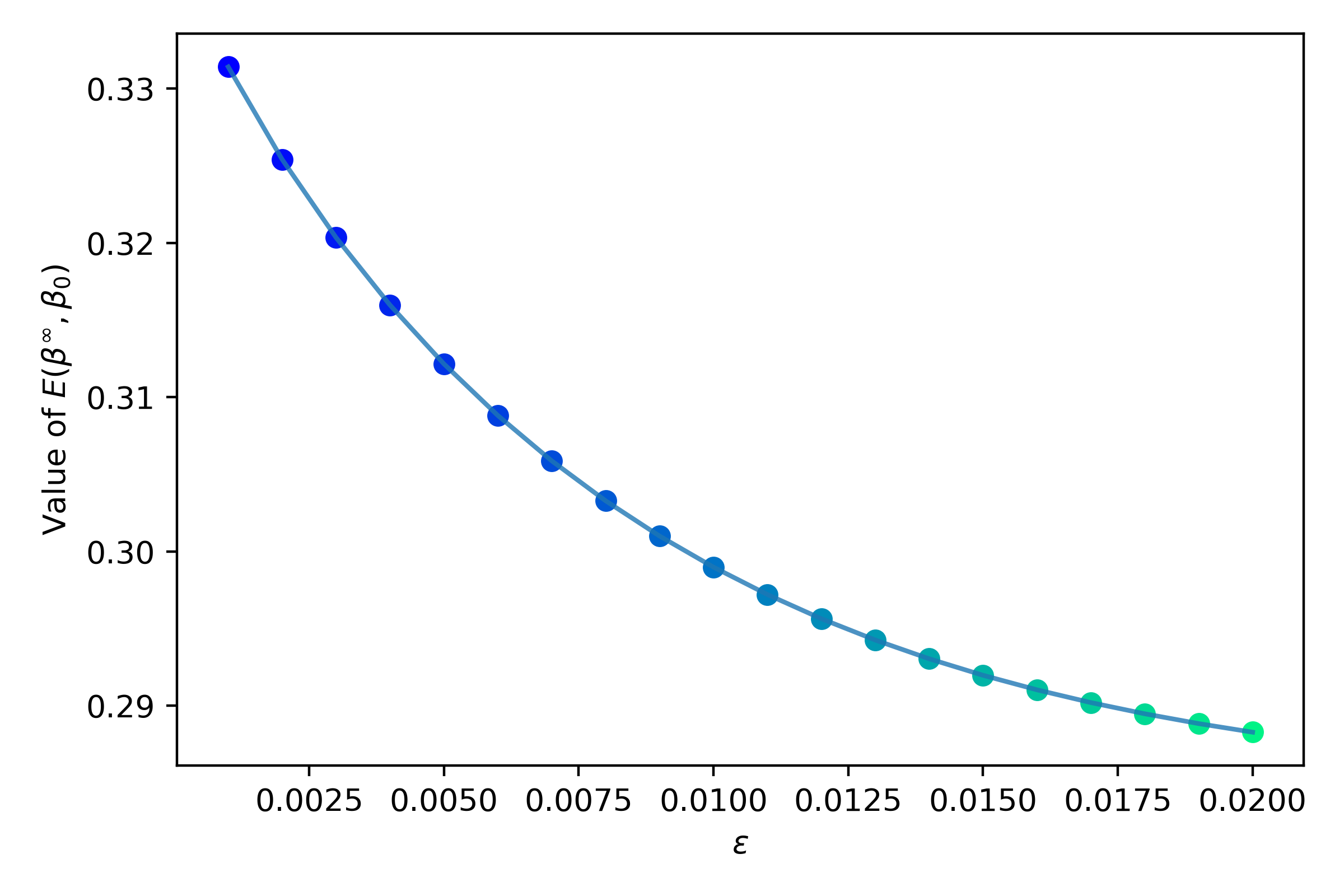}
    \caption{Bregman divergence style function value $E(\bm\beta^\infty, \bm\beta_0)$ of convergent solutions with different values of stability constant $\varepsilon$.}
    \label{fig:comparison-2}
\end{figure}

\paragraph{Extension to $N > 1$} We generalize the results to the case when $N > 1$ in the following corollaries. The proofs can be found in Appendix~\ref{appendixB}. We show that the convergent solution satisfies approximate KKT conditions of minimizing $E(\bm\beta, \bm\beta_0)$ among all solutions. Within a certain interval, a larger value of $\varepsilon$ leads to a greater reduction of the KKT error. The implications for tuning the stability constant $\varepsilon$ still hold.
\begin{corollary}\label{corollary:NDcorollary}
    For $N > 1$, let us suppose Assumption~\ref{assumption:data2} is satisfied. Then, as $t \to \infty$, the regression parameter converges to an interpolating solution $\bm\beta^\infty$ that satisfies the $\bar{\delta}$-KKT conditions for \begin{equation*}
            \underset{{\bm\beta \in \mathbb{R}^D \text{ s.t. } X\bm\beta = \bm y}}{\min} E(\bm\beta, \bm\beta_{0})
        \end{equation*}
    with the error $\bar{\delta}(\varepsilon)$ bounded by $ \sum_{n=1}^N \max\left\{\Big|M^{(n)}_+\Big|,~\Big|M_-^{(n)}\Big|\right\}$, where
    \begin{align*}
        M^{(n)}_+ := &\left( |\cos\theta_n| - |\sin\theta_n| \right) \left(\lambda_n^{-\frac{1}{4}}|y^{(n)}|^\frac{1}{2} - \frac{\sqrt{2}\varepsilon}{4\lambda_n^\frac{1}{2}|y^{(n)}|}\right),\\
        M^{(n)}_- := &\left(|\cos\theta_n|-|\sin\theta_n|\right)\left(\left(2\lambda_n\right)^{-\frac{1}{4}}|y^{(n)}|^{\frac{1}{2}}-\alpha\right) \\
        &- 2\sqrt{\frac{2\varepsilon}{\lambda_n^{\frac{3}{4}}|\sin\theta_n||y^{(n)}|^{\frac{1}{2}}}} \\
        &- \frac{3\sqrt{2}\varepsilon}{\lambda_n^{\frac{1}{2}}|\sin\theta_n y^{(n)}|}\ln \left(\frac{\lambda_n^{\frac{1}{4}}|\sin\theta_n||y^{(n)}|^{\frac{3}{2}}}{\sqrt{2}\varepsilon}\right).
    \end{align*}
\end{corollary}

\begin{corollary}\label{corollary:NDvarepsilon}
There exists a non-degenerate interval $\mathcal{J} \subseteq \mathcal{I}_\varepsilon$ such that for all $\varepsilon \in \mathcal{J}$, \begin{align*}
        \bar{\delta}(\varepsilon) \leq &\sum_{n=1}^N \left( |\cos\theta_n| - |\sin\theta_n| \right) \left(\lambda_n^{-\frac{1}{4}}|y^{(n)}|^\frac{1}{2}\right)  - \\&\left(\sum_{n=1}^N\left( |\cos\theta_n| - |\sin\theta_n| \right)\frac{\sqrt{2}}{4\lambda_n^\frac{1}{2}|y^{(n)}|}\right)\varepsilon.
    \end{align*}
\end{corollary}

\paragraph{Extension to Higher Order Models} Our analysis is generalizable to parameterizations with higher order $H \geq 2$ in weights, given by $\bm{\beta} = \bm{u}^H - \bm{v}^H$. Here $s^H$ denotes applying Hadamard product $H$ times on vector $s$. This parameterization can be interpreted as a diagonal linear neural network of depth $H$, as explained in \citep{woodworth2020kernel}. The mirror map is induced by a potential function closely related to $l_{2 - \frac{1}{H}}$-norm of $\bm\beta$, given by $\Phi^H_{t}(\bm\beta) := \sum_{i=1}^D \left(|\beta_i| + v^H_{i, t}\right)^{2-\frac{1}{H}}$, where $v_{i, t} = \mathcal{O}(\varepsilon)$. When the depth $H \to \infty$, the potential function approximates the squared $l_2$-norm.

\section{Conclusion}
In this work, we propose an MD perspective of the dynamics of smoothed sign descent for overparameterized regression problems. We extend existing results beyond GD to a case where update directions deviate from true gradients due to adaptivity, and formulate the equivalent dual dynamics with a simplified structure. We also study the role of the stability constant $\varepsilon$ in bounding the deviation of the convergent solution from minimizing a Bregman divergence style function. The finding supports the benefit of tuning the stability constant $\varepsilon$.

\bibliography{reference}
\bibliographystyle{abbrvnat}

\appendix
\onecolumn
\section{Proof of Results in Section~\ref{section:threestages}}\label{appendixA}
\textbf{Notation}.
Assumption~\ref{assumption:data} guarantees that the non-zero entries in $X$ are non-overlapping across rows. Therefore, we can partition the index set $I = \{1, \dots, D\}$ into $N$ disjoint subsets $I^{(1)}, \dots, I^{(N)}$ such that \begin{align}
    I = \bigcup_{n=1}^N I^{(n)},~I^{(n)} := \left\{ i \in [D] : x^{(n)}_i \neq 0 \right\}.
\end{align}
We define $\bm{w}^{+(n)}, \bm{w}^{-(n)}, \bm\beta^{(n)} \in \mathbb{R}^{D_n}$ as the subvectors of $\bm w^+$, $\bm w^-$ and $\bm\beta$ corresponding to the indices in $I^{(n)}$, respectively. Similarly, we define $\bm g^{+(n)}, \bm g^{-(n)}$ as subvectors of gradients $\nabla_{\bm w^+}L(\bm w), \nabla_{\bm w^-}L(\bm w)$ corresponding to the indices in $I^{(n)}$.
We let $\bm w^{(n)} := \begin{bmatrix}
    \bm w^{+(n)}, & \bm w^{-(n)}
\end{bmatrix} \in \mathbb{R}^{2D_n}$ and $\bm g^{(n)} := \begin{bmatrix}
    \bm g^{+(n)}, & \bm g^{-(n)}
\end{bmatrix} \in \mathbb{R}^{2D_n}$. Then
the weight dynamics \eqref{eq:weightdynmics} can be decomposed into $N$ autonomous ODE systems:
\begin{equation}
    \frac{d\bm w^{(n)}(t)}{dt} = F^{(n)}\left(\bm w^{(n)}(t)\right) := -\frac{\bm g^{(n)}(t)}{\bm g^{(n)}(t) + \varepsilon \bm 1},
\label{eq:odesystem1block}\end{equation}
where $\bm w^{(n)}(0) = \alpha \bm{1}$ for each $n$. The residual for each $n$ is defined by $r^{(n)}(t) := y^{(n)} - \sum_{i=1}^{D_n} x^{(n)}_i \beta^{(n)}_i(t)$.
In this section, we prove the results for an arbitrary $n$. We omit the superscripts $(n)$ when possible to simplify the notation. 

\subsection{Proof of Proposition~\ref{proposition:signandmonotone}}
\begin{proof}
For all $i=1, \dots, D_n$, it is easy to see that $g^+_i(t) = -w_i^+(t) \cdot x_i \cdot r(t)$, $g_i^-(t) = w_i^-(t)\cdot x_i \cdot r(t)$. The dynamics follow
$w^+_i(t)' = -\frac{g_i^+(t)}{|g_i^+(t)| + \varepsilon},~w^-_i(t)' = -\frac{g_i^-(t)}{|g_i^-(t)| + \varepsilon}$. 

First, we show that for all $i$, $w^+_i(t), w^-_i(t) \geq 0$ always hold. Suppose for contradiction that $w^+_i(t') < 0$ for some $t'$. Since $w^+_i(0) = w^-_i(0) = \alpha > 0$, by continuity of $w^+_i(t)$, there exists $t_0 \in (0, t')$ such that $w^+_i(t_0) = 0$ and $w^+_i(t_0)' < 0$. However, $w^+_i(t_0) = 0$ implies $g^+_i(t_0) = 0$ and $w^+_i(t_0)' = 0$. Therefore, $w^+_i(t)$ never changes sign and is always non-negative. Similarly, we can show that $w^-_i(t)$ is always non-negative.

Next, we show that for each $i$, if $w^+_i(0)' > 0$, then $w^+_i(t)' \geq 0, w^-_i(t) \leq 0$. Relation $w^+_i(0)' = \alpha x_i y > 0$ implies that $x_i y > 0$. Then, $x_i r(0) = x_i y > 0$. Let us suppose for contradiction that there exists $t' > 0$ such that $x_i r(t') < 0$. By continuity of $x_i r(t)$, there exists $t_0 \in (0, t')$ such that $x_i r(t_0) = 0$. Since $x_i \neq 0$ by assumption, we must have $r(t_0) = 0$. In turn, $x_j r(t_0) = 0$ and $g^+_j(t_0) = g^-_j(t_0) = 0$ for all $j = 1, \dots, D_n$. As a result, $F^{(n)}\left(\bm w^{(n)}(t_0)\right) = \bm {0}$ and $\bm w^{(n)}(t_0)$ is an equilibrium of the autonomous ODE system \eqref{eq:odesystem1block}. Then, for all $t \geq t_0$, $\bm w^{+(n)}(t) = \bm w^{+(n)}(t_0)$ and $\bm w^{-(n)}(t) = \bm w^{-(n)}(t_0)$. Therefore, we get $x_i r(t) = x_i r(t_0) = 0$ for all $t \geq t_0$. However, this contradicts that $x_i r(t') < 0$ and $t' > t_0$. Thus, we must have $x_i r(t) \geq 0$ for all $t \geq 0$. Since $w^+_i(t), w^-_i(t) \geq 0$, it follows that $g^+_i(t) \leq 0$ and $g^-_i(t) \geq 0$ for all $t$. Hence, we get $w^+_i(t)'
\geq 0$ and $w^-_i(t)' \leq 0$ for all $t$.

If $w^+_i(0)' = \alpha x_i y \leq 0$, since $x_i$ and $y$ are non-zero by assumption, we must have $x_i y < 0$. Using similar arguments, it follows that $w^+_i(t)' \leq 0$ and $w^-_i(t)' \geq 0$ for all $t$.
\end{proof}

\begin{lemma}\label{lemma:residualmonotone}
Residual $r(t)$ never changes sign and its absolute value is always non-increasing.
\end{lemma}

\begin{proof}
    We have $r(0) = y \neq 0$ by assumption. When $r(0) > 0$, suppose for contradiction that there exists $t' > 0$ such that $r(t') < 0$. By continuity, we must have $r(t_0) = 0$ for some $t_0 \in (0, t')$. Then, we have $g^+_i(t_0) = g^-_i(t_0) = 0$ for all $i = 1, \dots, D_n$. In turn, $w^+_i(t_0)' = w^-_i(t_0)' = 0$ for all $i$ and $\bm w^{(n)}(t_0)$ is an equilibrium of the autonomous ODE system \eqref{eq:odesystem1block}. Therefore, $w^+_i(t) = w^+_i(t_0)$ and $w^-_i(t) = w^-_i(t_0)$ for all $t \geq t_0$. We conclude $r(t) = r(t_0) = 0$ for all $t \geq t_0$. This contradicts that $r(t') < 0$ for $t' > t_0$. As a result, $r(t) \geq 0$ for all $t$. Similarly, when $r(t) < 0$, it follows that $r(t) \leq 0$ for all $t$.

    Next, we compute the derivative of $r(t)$ with respect to $t$ as \begin{align*}
        r'(t) &= -2\sum_{i=1}^{D_n} x_i \left(w^+_i(t)\cdot w^+_i(t)' - w^-_i(t) \cdot w^-_i(t)'\right) \\
        &= -2\sum_{i=1}^{D_n} x_i \left(w^+_i(t)\cdot \frac{w^+_i(t) x_i r(t)}{|w^+_i(t) x_i r(t)| + \varepsilon} - w^-_i(t)\cdot \frac{-w^-_i(t) x_i r(t)}{|w^-_i(t) x_i r(t)| + \varepsilon}\right) \\
        &= -2\sum_{i=1}^{D_n} x_i^2\left(\frac{(w^+_i(t))^2}{|w^+_i(t) x_i r(t)| + \varepsilon} + \frac{(w^-_i(t))^2}{|w^-_i(t) x_i r(t)| + \varepsilon}\right)r(t).
    \end{align*}
    Notice that $x_i^2\left(\frac{(w^+_i(t))^2}{|w^+_i(t) x_i r(t)| + \varepsilon} + \frac{(w^-_i(t))^2}{|w^-_i(t) x_i r(t)| + \varepsilon}\right) \geq 0$. When $r(0) > 0$, we have shown that $r(t) \geq 0$ for all $t$. Then, $r'(t) \leq 0$ for all $t$. Similarly, when $r(0) < 0$, we have $r(t) \leq 0$ and $r'(t) \geq 0$ for all $t$. Hence, the magnitude of the residual $r(t)$ is always non-increasing.
\end{proof}

Following the notation in Section~\ref{section:threestages}, we let $u_i$ denote the dominating weight, and let $v_i$ represent the non-dominating weight, i.e.,
\begin{align*}
    u_i(t) &:= \begin{cases}
        w_i^+(t) & \text{ if } w_i^+(0)' > 0, \\
        w_i^-(t) & \text{else},
    \end{cases}\\
    v_i(t) &:= \begin{cases}
        w_i^-(t) & \text{ if } w_i^+(0)' > 0, \\
        w_i^+(t) & \text{else}.
    \end{cases}
\end{align*}
If $x_i y > 0$, then $\beta_i = u_i^2 - v_i^2$; if $x_i y < 0$, then $\beta_i = -u_i^2 + v_i^2$. Therefore, for all $i$, \begin{align}
    \beta_i(t) &= \operatorname{sgn}(x_i y)\left(u_i^2(t) - v_i^2(t)\right).\label{eq:signofbeta}
    \end{align}
We define \begin{align*}
    f_i(t) &:= -u_i(t)|x_i r(t)|,~h_i(t) := v_i(t)|x_i r(t)|.
    \end{align*}
Then, \begin{align*}
    u'_i(t) &= -\frac{f_i(t)}{|f_i(t)| + \varepsilon} = \frac{u_i(t) |x_i r(t)|}{u_i(t) |x_i r(t)| + \varepsilon},\\
    v'_i(t) &= -\frac{h_i(t)}{|h_i(t)| + \varepsilon} = - \frac{v_i(t) |x_i r(t)|}{v_i(t) |x_i r(t)| + \varepsilon}.
\end{align*}
The residual can be written as \begin{align}
    r(t)  &= y - \sum_{k=1}^{D_n} x_k\beta_i(t) \label{eq:expressionofrtinbeta} \\
    &= y - \sum_{k=1}^{D_n} \operatorname{sgn}(y)\operatorname{sgn}(x_k)\cdot x_k\left(u_k^2(t) - v_k^2(t) \right) \\
    &= \operatorname{sgn}(y)\left(|y| - \sum_{k=1}^{D_n} |x_k| \left(u_k^2(t) - v_k^2(t) \right)\right).
\end{align}
By Lemma~\ref{lemma:residualmonotone}, $r(t)$ never changes sign. Since $r(0) = \operatorname{sgn}(y)|y|$, then for all $t$, \begin{equation}\label{eq:absresidual}
    |r(t)| = |y| - \sum_{k=1}^{D_n} |x_k|\left(u_k^2(t) - v_k^2(t) \right).
\end{equation}
\subsection{Proof of Proposition~\ref{proposition:initialstage}}
\begin{proof}
    First, we show the existence of $t_i > 0$ such that $h_i(t_i) = \varepsilon$ for each $i$. By Assumption~\ref{assumption:epsilonalpha}, $h_i(0) = \alpha|x_i y| \geq 2\varepsilon$. Let us suppose for contradiction that $h_i(t) > \varepsilon$ for all $t$. Then $v'_i(t) < -\frac{\varepsilon}{\varepsilon + \varepsilon} = -\frac{1}{2}$ for all $t$, and for $t > 2\alpha$, $v_i(t) < \alpha - \frac{1}{2} t < 0$. However, by Proposition~\ref{proposition:signandmonotone}, $v_i(t)$ is always non-negative. It yields $h_i(t_i') < \varepsilon$ for some $t_i'$. Since $h_i(0) \geq 2\varepsilon$, by continuity of $h_i(t)$, there exists $t_i$ such that $h_i(t_i) = \varepsilon$. It follows that $t_i \leq 2\alpha$. Because $h_2(t) \geq \varepsilon$ for $t \leq t_i$, $v_2'(t) \leq -\frac{1}{2}$. If $t_i > 2\alpha$, then $v_2(t_i) < \alpha - \frac{1}{2} t_i < 0$, which is a contradiction. We conclude $t_i \in (0, 2\alpha]$.

    Next, we show that $|r(t_i)|$ is lower bounded. Using \eqref{eq:absresidual}, we get \begin{align*}
        |r(t)| &= |y| - \sum_{i=1}^{D_n} |x_i| \left(u_i^2(t) - v_i^2(t)\right)\\
        &\geq |y| - \sum_{i=1}^{D_n} |x_i| u_i^2(t).
    \end{align*}
    For all $i$, $u_i'(t) \leq 1$ always holds. It follows that $u_i(t) \leq \alpha + t$ for all $t$. Since $|r(t)|$ is non-increasing by Lemma~\ref{lemma:residualmonotone}, and using $t_i \leq 2\alpha$, it follows that
    \begin{equation*}
        |r(t_i)| \geq |r(2\alpha)| \geq |y| - \sum_{i=1}^{D_n} |x_i| \left(\alpha + 2\alpha\right)^2 = |y| - 9\alpha^2 \sum_{i=1}^{D_n} |x_i|.
    \end{equation*}
    By Assumption~\ref{assumption:epsilonalpha}, $9\alpha^2 \leq \frac{|y|}{2\sum_{i=1}^{D_n} |x_i|}$, and then \begin{equation}
        |r(t_i)| \geq |r(2\alpha)| \geq |y| - \frac{|y|}{2} = \frac{|y|}{2}.\label{eq:lowerboundresidualat2alpha}
    \end{equation} 
    Since $h_i(t_i) = v_i(t_i)|x_i r(t_i)| = \varepsilon$, we get $v_i(t_i) \leq \frac{2\varepsilon}{|x_i y|}$. Function $v_i(t)$ is non-increasing by Proposition~\ref{proposition:signandmonotone}, so for all $t \geq t_i$, we have $v_i(t) \leq \frac{2\varepsilon}{|x_i y|}$. The argument holds for all $i = 1, \dots, D_n$ and for all $n$. We complete the proof by letting $T_0 := \max \{ t_i \}$.
\end{proof}

\subsection{Proof of Proposition~\ref{proposition:dualdynamics}}
\begin{proof}
    Let us define the potential function by $\Phi_t(\bm \beta(t)) := \frac{2}{3}\sum_{i=1}^{D}\left(|\beta_i(t)| + v^2_{i, t}\right)^{\frac{3}{2}}$ for all $t$, where $v_{i. t} := v_i(t)$ is a parameter for the time-varying potential. We get the dual variable using the mirror map \begin{equation*}
        \nabla \Phi_t(\bm\beta(t)) = \operatorname{sgn}(\bm \beta (t))\odot \left(|\bm\beta(t)| + \bm v^2_t\right)^{\frac{1}{2}},
    \end{equation*} where operations are taken element-wise. The Hessian $\nabla^2 \Phi_t(\bm\beta(t))$ is a diagonal matrix with diagonal elements $\frac{\operatorname{sgn}(\beta_i(t))}{2\left(|\beta_i(t)| + v^2_{i,t}\right)^{\frac{1}{2}}}$. 
    Using the chain rule, we compute the dual dynamics \begin{align*}
        \frac{d\nabla \Phi_t(\bm \beta(t))}{dt} = &\langle \nabla^2 \Phi_t(\bm\beta(t)), \frac{d\bm \beta(t)}{dt} \rangle + \langle \nabla_{\bm v} \nabla \Phi_t(\bm\beta(t)), \frac{d\bm v(t)}{dt}\rangle \\
        = &\operatorname{sgn}(\bm\beta(t)) \odot \left(|\bm\beta(t)| + \bm v_t^2\right)^{-\frac{1}{2}} \odot \left( \bm u(t) \odot \frac{d\bm u(t)}{dt} - \bm v(t) \odot \frac{d\bm v(t)}{dt}\right) \\
        &+ \operatorname{sgn}(\bm \beta(t)) \odot \left( |\bm \beta (t)| + \bm v^2_t \right)^{-\frac{1}{2}} \odot \bm v(t) \odot \frac{d\bm v(t)}{dt} \\
        = &\operatorname{sgn}(\bm\beta(t)) \odot \frac{d\bm u(t)}{dt} \\
        = &-\operatorname{sgn}(\bm\beta(t)) \odot \frac{\nabla_{\bm u}L(\bm w(t))}{\left|\nabla_{\bm u}L(\bm w(t))\right|+\varepsilon \bm{1}}.
    \end{align*}
\end{proof}

\begin{lemma}\label{lemma:xixjcomparison}
    For all $i, j \in \{1, \dots, D_n \}$, $|x_i| \geq |x_j|$ implies $u_i(t) \geq u_j(t)$ for all $t \geq 0$.
\end{lemma}

\begin{proof}
    If $|x_i| = |x_j|$, since $u_i(0) = u_j(0) = \alpha$, then $u'_i(t) = u'_j(t)$ and $u_i(t) = u_j(t)$. Suppose $|x_i| > |x_j|$. Let $\bar{u}(t) := u_i(t) - u_j(t)$. Then we have $\bar{u}(0) = \alpha - \alpha = 0$.
    
    First, we show that there exists a small neighborhood $\mathcal{B}$ such that $\bar{u}(t) > 0$ for $t \in \mathcal{B}$. Because $u_i(t), u_j(t)$ are differentiable everywhere, $\bar{u}(t)$ is differentiable for all $t \geq 0$. Inequality $|x_i| > |x_j|$ implies $u_i'(0) = \frac{\alpha |x_i y|}{\alpha |x_i y| + \varepsilon} > \frac{\alpha |x_j y|}{\alpha|x_j y| + \varepsilon} = u_j'(0)$. As a result, $\bar{u}'(0) > 0$. Using differentiability of $\bar{u}(t)$ at $t=0$, we get
    \begin{equation}
        \lim_{\tau\to 0^+} \frac{\bar{u}(\tau) - \bar{u}(0)}{\tau - 0} = \lim_{\tau\to 0^+} \frac{\bar{u}(\tau)}{\tau} = \bar{u}'(0).\label{eq:limat0}
    \end{equation}
   Let $\epsilon_{\tau} := \frac{\bar{u}'(0)}{3} > 0$. By definition of limit in \eqref{eq:limat0}, there exists $\delta_{\tau} > 0$ such that for all $\tau \in (0, \delta_{\tau})$, $\left |\frac{\bar{u}(\tau)}{\tau} - \bar{u}'(0)\right | < \epsilon_{\tau}$. Therefore, $\frac{\bar{u}(\tau)}{\tau} - \bar{u}'(0) > -\epsilon_\tau = -\frac{\bar{u}'(0)}{3}$. Then, $\bar{u}(\tau) > \frac{2\tau}{3} \bar{u}'(0) > 0$ for all $\tau \in (0, \delta_\tau)$.
    
    Next, we show that $\bar{u}(t) \geq 0$ for all $t > 0$. Suppose for contradiction that there exists $t > 0$ such that $\bar{u}(t) < 0$. Let $t_0 := \inf\{ t : t > 0,~\bar{u}(t) < 0\}$. Then $\bar{u}(t_0) \leq 0$ and $\bar{u}(t) \geq 0$ for $t \in (0, t_0)$ by definition. We must have $t_0 \geq \delta_{\tau} > \frac{\delta_{\tau}}{2} > 0$ as we have shown that $\bar{u}(t) > 0$ for $t \in (0, \delta_{\tau})$. Since $\bar{u}(t)$ is differentiable, by the Mean Value Theorem, there exists $t_1 \in (\frac{\delta_{\tau}}{2}, t_0)$ such that $\bar{u}'(t_1) = \frac{\bar{u}(t_0) - \bar{u}(\frac{\delta_{\tau}}{2})}{t_0 - \frac{\delta_{\tau}}{2}}$. Since $\bar{u}(t_0) \leq 0$ and $\bar{u}(\frac{\delta_{\epsilon}}{2}) > 0$, we have $\bar{u}'(t_1) < 0$. Therefore,
    \begin{align*}
        \bar{u}'(t_1) &= u_i'(t_1) - u_j'(t_1) \\
        &= \frac{u_i(t_1)|x_i r(t_1)|}{ u_i(t_1)|x_i r(t_1)| + \varepsilon} -\frac{u_j(t_1)|x_j r(t_1)|}{ u_j(t_1)|x_j r(t_1)| + \varepsilon} \\
        &< 0.
    \end{align*}
    The inequality implies that $u_i(t_1)|x_i r(t_1)|  < u_j(t_1) |x_j r(t_1)|$. By assumption, we have $|x_i| > |x_j|$, and then we must have $u_i(t_1) < u_j(t_1)$, i.e., $\bar{u}(t_1) < 0$. However, $t_1 < t_0$ and this contradicts that $\bar{u}(t) \geq 0$ for all $t \in (0, t_0)$. Thus, $u_i(t) \geq u_j(t)$ always holds.
\end{proof}

\subsection{Proof of Proposition~\ref{proposition:mainstage}}
\begin{proof}
First, we show that for all $i$, there exists $T_i$ such that $\left|\nabla_{\bm u}L(\bm w(T_i))\right|_i = f_i(T_i) = \varepsilon$. 
Suppose for contradiction that $f_i(t) > \varepsilon$ for all $t$. Then, $u'_i(t) > \frac{1}{2}$ and $u_i(t) \geq \alpha + \frac{t}{2}$. Without loss of generality, we assume $r(0) = y > 0$. For $t > 2\sqrt{\frac{|y|}{|x_i|}}$, the residual is negative due to \begin{align*}
    r(t) &= y - |x_i|\left(u_i^2(t) - v_i^2(t)\right) - \sum_{k \neq i}^{D_n} |x_k|\left(u_i^2(t) - v_i^2(t)\right) \\
    &\leq y - |x_i|\left(u_i^2(t) - v_i^2(t)\right) \\
    &\leq y - |x_i|\left(\left(\alpha + \sqrt{\frac{|y|}{|x_i|}}\right)^2 - \alpha^2\right) \\
    &< y - |x_i| \frac{|y|}{|x_i|} \\
    &= 0.
\end{align*}
However, this contradicts that $r(t)$ never flips sign by Lemma~\ref{lemma:residualmonotone}. Hence, there exists $T_i'$ such that $f_i(T_i') \leq \varepsilon$. By continuity, there exists $t \in (0, T_i']$ such that $f_i(t) = \varepsilon$. Let $T_i := \min\left\{t : 0 \leq t \leq T'_i,~ f_i(t) =\varepsilon\right\}$. Therefore, \begin{equation}
    f_i(T_i) = \varepsilon,\text{ and }f_i(t) > \varepsilon~\text{for }t < T_i.\label{eq:fT_icomparison}
\end{equation}

Next, we show that $T_i > 2\alpha \geq T_0$. In \eqref{eq:lowerboundresidualat2alpha} we have proved that $|r(2\alpha)| \geq \frac{|y|}{2}$. Since $u_i(t) \geq \alpha$ and $|r(t)| \geq |r(2\alpha)|$ for $t \leq 2\alpha$, then $f_i(t) = u_i(t) |x_i r(t)| \geq \alpha \frac{|x_i y|}{2}$. By Assumption~\ref{assumption:epsilonalpha}, we have $\alpha > \frac{2\varepsilon}{|x_i y|}$. As a result, $f_i(t) > \varepsilon$ for all $t \leq 2\alpha$. Therefore, we must have $T_i' > 2\alpha \geq T_0$.

We need to show that the derivative of $f_i(t)$ is always non-positive for $t \geq T_i$. Using the expression for $|r(t)|$ in \eqref{eq:absresidual}, we get
\begin{align*}
    f_i'(t) &= |x_i|\left(u_i'(t) |r(t)| + u_i(t)|r(t)|'\right) \\
    &= |x_i|\left(u'_i(t)|r(t)| + u_i(t)\left(-2\sum_{k=1}^{D_n} |x_k|u_k(t)u'_k(t) + 2\sum_{k=1}^{D_n} |x_k|v_k(t)v'_k(t)\right)\right). 
\end{align*}
Since $v'_k(t) \leq 0$ for all $k$, we get
\begin{equation}\label{eq:fderivativeform1}
    f_i'(t) \leq |x_i|\left(u'_i(t)|r(t)| - 2u_i(t)\sum_{k=1}^{D_n} |x_k|u_k(t)u'_k(t)\right).
\end{equation}

Next, we want to find a lower bound for $2u_i(T_i)\sum_{k=1}^{D_n} |x_k|u_k(T_i)u'_k(T_i)$. We denote the index set by $\mathcal{I} := \left \{1, \dots, D_n\right \}$ that we partition as $\mathcal{I} = \mathcal{I}_i^+ \cup \mathcal{I}_i^-$, where $\mathcal{I}_i^+ := \left\{ k : |x_k| \geq |x_i| \right\}$ and $\mathcal{I}_i^- := \left\{ k : |x_k| < |x_i| \right\}$. 
For $k \in \mathcal{I}_i^-$, since $|x_k| < |x_i|$, we have $u_k(t) \leq u_i(t)$ by Lemma~\ref{lemma:xixjcomparison}. Therefore, $u_k(t)|x_kr(t)| \leq u_i(t)|x_ir(t)|$. Then, for all $k \in \mathcal{I}_i^-$ and for all $t$, \begin{align}
    u_i(t) u_k'(t) &= u_i(t) \frac{u_k(t)|x_k r(t)|}{u_k(t)|x_k r(t)| + \varepsilon} \\
    &\geq u_i(t) \frac{u_k(t)|x_k r(t)|}{u_i(t)|x_i r(t)| + \varepsilon} \\
    &\geq \frac{u_k(t)|x_k|}{|x_i|}\cdot \frac{u_i(t)|x_i r(t)|}{u_i(t)|x_i r(t)| + \varepsilon}\\
    &= \frac{|x_k|}{|x_i|}u_k(t)u'_i(t) \label{eq:inequalityforI-}.
\end{align}
For $k \in \mathcal{I}_i^+$, similarly, we have $u_k(t) \geq u_i(t)$ for all $t$. We do have $f_k(t) = u_k(t)|x_k r(t)| \geq u_i(t)|x_i r(t)| = f_i(t)$. Then, $u'_k(t) \geq u_i'(t)$. Therefore, for all $k \in \mathcal{I}_i^+$ and for all $t$, we have
\begin{equation}
    u_i(t)u_k'(t) \geq u_i(t)u_i'(t)\label{eq:inequalityforI+1}.
\end{equation}
Then, using \eqref{eq:inequalityforI-} and \eqref{eq:inequalityforI+1}, we get
\begin{align*}
    2u_i(t)\sum_{k=1}^{D_n} |x_k|u_k(t)u_k'(t) 
    &= 2\sum_{k\in\mathcal{I}_i^+} |x_k|u_k(t) u_i(t)u_k'(t) + 2\sum_{k\in\mathcal{I}_i^-} |x_k|u_k(t)u_i(t)u_k'(t) \\
    &\geq 2\sum_{k\in\mathcal{I}_i^+} |x_k|u_k(t) u_i(t)u_i'(t) + 2\sum_{k\in\mathcal{I}_i^-} \frac{|x_k|}{|x_i|}\cdot|x_k|u_k^2(t)u_i'(t) \\
    &= 2u'_i(t) \left( \sum_{k\in\mathcal{I}_i^+} |x_k|u_i(t)u_k(t) + \sum_{k\in\mathcal{I}_i^-} \frac{|x_k|}{|x_i|}|x_k|u_k^2(t)\right).
\end{align*}
Because $u_i(t)$ is non-decreasing for all $i$, it follows that for $t \geq T_i$,
\begin{align*}
    2u'_i(t) \left( \sum_{k\in\mathcal{I}_i^+} |x_k|u_i(t)u_k(t) + \sum_{k\in\mathcal{I}_i^-} \frac{|x_k|}{|x_i|}|x_k|u_k^2(t)\right) &\geq 2u'_i(t) \left( \sum_{k\in\mathcal{I}_i^+} |x_k|u_i(T_i)u_k(T_i) + \sum_{k\in\mathcal{I}_i^-} \frac{|x_k|}{|x_i|}|x_k|u_k^2(T_i)\right).
\end{align*}
Moreover, for $t \in [0, T_i]$, $u'_i(t) \geq \frac{1}{2}$ and $u'_k(t) \leq 1$. As a result, $u'_i(t) \geq \frac{1}{2}u'_k(t)$. Then, \begin{equation}
    u_i(T_i) \geq \alpha + \frac{1}{2}\left(u_k(T_i) - \alpha \right) > \frac{1}{2}u_k(T_i).\label{eq:uigreaterthanuk}
\end{equation}
We also know that $\frac{|x_k|}{|x_i|} \geq \frac{\min_j \{ |x_j| \}}{|x_i|}$ for all $k \in \mathcal{I}_i^-$, and $1 \geq \frac{\min_j \{ |x_j| \}}{|x_i|}$. Using \eqref{eq:uigreaterthanuk}, we get \begin{align}
     2u_i(t)\sum_{k=1}^{D_n} |x_k|u_k(T_i)u_k'(T_i) &\geq u'_i(t) \left( \sum_{k\in\mathcal{I}_i^+} 2|x_k|\frac{1}{2}u_k(T_i)u_k(T_i) + \sum_{k\in\mathcal{I}_i^-} \frac{2|x_k|}{|x_i|}|x_k|u_k^2(T_i)\right) \\
     &= u'_i(t) \left( \sum_{k\in\mathcal{I}_i^+} |x_k|u_k^2(T_i) + \sum_{k\in\mathcal{I}_i^-} \frac{2|x_k|}{|x_i|}|x_k|u_k^2(T_i)\right) \\
     &\geq u'_i(t) \left( \sum_{k\in\mathcal{I}_i^+} \frac{\min_j \{ |x_j| \}}{|x_i|}|x_k|u_k^2(T_i) + \sum_{k\in\mathcal{I}_i^-} \frac{\min_j \{ |x_j| \}}{|x_i|}|x_k|u_k^2(T_i)\right) \\
     &= u'_i(t) \frac{\min_j\{|x_j|\}}{|x_i|}\sum_{k=1}^{D_n}|x_k|u_k^2(T_i).\label{eq:lowerboundforsum}
\end{align}

Next we consider \eqref{eq:fderivativeform1} by using \eqref{eq:lowerboundforsum}. Since $|r(t)|$ is non-increasing, for all $t \geq T_i$, we have
\begin{align}
    f'_i(t) &\leq |x_i|\left(u'_i(t)|r(t)| - u'_i(t) \frac{\min_j\{|x_j|\}}{|x_i|}\sum_{k=1}^{D_n}|x_k|u_k^2(T_i)\right) \\
    &\leq |x_i|\left(u'_i(t)|r(T_i)| - u'_i(t) \frac{\min_j\{|x_j|\}}{|x_i|}\sum_{k=1}^{D_n}|x_k|u_k^2(T_i)\right) \\
    &\leq |x_i|u'_i(t)\left(|r(T_i)| - \frac{\min_j\{|x_j|\}}{|x_i|}\sum_{k=1}^{D_n}|x_k|u_k^2(T_i)\right).\label{eq:proposition36finalinequality2}
\end{align}

At $t = T_i$, we know that $u_i(T_i) \geq \alpha$ and $f_i(T_i) = u_i(T_i)|x_i r(T_i)| = \varepsilon$. By Assumption~\ref{assumption:epsilonalpha}, we have $\alpha > \frac{2\varepsilon}{|x_j y|}$ for all $j$. We must have 
\begin{equation}
    |r(T_i)| = \frac{f_i(t)}{u_i(t) |x_i|} \leq \frac{\varepsilon}{\alpha |x_i|} < \frac{\varepsilon}{|x_i|}\cdot \frac{\min_j\{|x_j|\}|y|}{2\varepsilon} = \frac{1}{2}\frac{\min_j\{|x_j|\}}{|x_i|}|y|\label{eq:proposition36finalinequality3},
\end{equation}
which implies
\begin{equation*}
    |y| - \sum_{k=1}^{D_n} |x_k|\left( u_k^2(T_i) - v_k^2(T_i) \right) < \frac{1}{2}\frac{\min_j\{|x_j|\}}{|x_i|}|y|.
\end{equation*}
Thus, \begin{equation}
    \sum_{k=1}^{D_n} |x_k|u_k^2(T_i) \geq \sum_{k=1}^{D_n} |x_k|\left( u_k^2(T_i) - v_k^2(T_i) \right) > \left(1 - \frac{1}{2}\frac{\min_j\{|x_j|\}}{|x_i|}\right)|y| \geq \frac{1}{2}|y|.\label{eq:proposition36finalinequality}
\end{equation}
By using \eqref{eq:proposition36finalinequality3} and \eqref{eq:proposition36finalinequality} in \eqref{eq:proposition36finalinequality2}, we get
\begin{align*}
    f_i'(t) &\leq |x_i|u'_i(t)\left( \frac{1}{2}\frac{\min_j\{|x_j|\}}{|x_i|}|y| - \frac{\min_j\{|x_j|\}}{|x_i|}\sum_{k=1}^{D_n}|x_k|u_k^2(T_i) \right) \\
    &\leq |x_i|u'_i(t)\left( \frac{1}{2}\frac{\min_j\{|x_j|\}}{|x_i|}|y| - \frac{\min_j\{|x_j|\}}{|x_i|}\frac{1}{2}|y| \right) \\
    &\leq 0.
\end{align*}

Hence, for all $t\geq T_i$, $f'_i(t)$ is non-increasing. We conclude that for each $i$, there exists $T_i > T_0$ such that $f'_i(t) > \varepsilon$ for $t < T_i$, and $f'_i(t) \leq \varepsilon$ for $t \geq T_i$.
\end{proof}
\begin{lemma}[Convergence]\label{lemma:convergence}
    As $t\to\infty$, for every $n$ we have \begin{align*}
        &\lim_{t \to \infty} r^{(n)}(t) = 0, \\
        &\lim_{t\to\infty}\nabla_{\bm w}L(\bm w(t)) = \bm 0, \\
        &\bm u^{\infty} := \lim_{t\to\infty}(\bm u(t)) \text{~with}~u_i^\infty < \infty~\forall i, \\
        &\bm v^{\infty} := \lim_{t\to\infty}(\bm v(t)) \text{~with}~v_i^\infty < \infty~\forall i.
    \end{align*}
\end{lemma}

\begin{proof}
    Without loss of generality, we assume $r(0) = y > 0$. By Lemma~\ref{eq:absresidual}, $r(t)$ is bounded below by $0$ and monotonically non-increasing in $t$. Therefore, $r(t)$ converges as $t \rightarrow \infty$ by calculus. Let $R_0 := \lim_{t \to \infty} r(t) \geq 0$. We want to show that $R_0 = 0$. Suppose for contradiction that $R_0 > 0$. Then, $r(t) \geq R_0 > 0$ for all $t \geq 0$.
    
    We first show that $u_k'(t)$ is bounded below by a positive number for all $k$. Since $u_k(t) \geq \alpha$ and $r(t) \geq R_0$ for all $t$, we have $f_k(t) = u_k(t)|x_k|r(t) \geq \alpha |x_k|R_0 > 0$. Therefore, for all $t \geq 0$,
    \begin{equation*}
        u_k'(t) = \frac{f_k(t)}{f_k(t) + \varepsilon} \geq \frac{\alpha |x_k|R_0}{\alpha |x_k|R_0 + \varepsilon} > 0.
    \end{equation*}
    Then, $u_k(t) \geq \alpha + t \cdot \frac{\alpha |x_k|R_0}{\alpha |x_k|R_0 + \varepsilon}$. 
    Recall that \begin{align*}
        r (t) &= y - \sum_{k=1}^{D_n}|x_k|\left(u_k^2(t) - v_k^2(t) \right) \\
        &\leq y - \sum_{k=1}^{D_n}|x_k|\left(u_k^2(t)- \alpha^2\right).
    \end{align*}
    As $t \to \infty$, $u_k^2(t) \rightarrow \infty$, and the quantity $\sum_{k=1}^{D_n}|x_k|u_k^2(t)$ grows unbounded. Then, $r(t) < 0$ for sufficiently large $t$. This contradicts that $r(t) \geq 0$ for all $t$ by Lemma~\ref{lemma:residualmonotone}. Thus, we must have $R_0 = \lim_{t \to \infty} r(t) = 0$.

    The argument holds for all $n$, so $\lim_{t\to\infty}r^{(n)}(t) = 0$ for all $n = 1, \dots, N$. Then, we have $\lim_{t\to\infty}[\nabla_{\bm u}L(\bm w(t))]_i = 0$ and $\lim_{t\to\infty}[\nabla_{\bm v}L(\bm w(t))]_i = 0$ for all $i$. It follows that $\lim_{t\to\infty}\nabla_{\bm w} L (\bm w(t)) = \bm 0$.

    Next, we show that the weights converge as $t\rightarrow \infty$. Without loss of generality, we suppose $r(0) = y > 0$. Because $r(t)$ never changes sign by Lemma~\ref{lemma:residualmonotone}, we have $0 \leq r(t) \leq y - \sum_{k=1}^{D_n}|x_k|\left(u_k^2(t)- \alpha^2\right)$. As a result, $u_k(t)$ is upper bounded. Since $u_k(t)$ is non-decreasing, we have $u_k^{\infty} := \lim_{t\to\infty}u_k(t) < \infty$ by calculus. Using similar argument for $v_k(t)$ which is non-increasing and bounded below by $0$, $v_k^{\infty} := \lim_{t\to\infty}v_k(t) < \infty$. The proof holds for all $k$ and all $n$. Therefore, $\bm u^{\infty} := \lim_{t\to\infty}\bm u(t)$ exists with $u_i^\infty < \infty$ for all $i$, and $\bm v^{\infty} := \lim_{t\to\infty}\bm v(t)$ exists with $v_i^\infty < \infty$ for all $i$.
\end{proof}

\section{Proof of Results in Section~\ref{section:characterization}}\label{appendixB}
Using Assumption~\ref{assumption:data2}, we parameterize the dynamics using $\theta_1$ and $\lambda_1$ with $|\cos\theta_1| \geq |\sin\theta_1| > 0$ and $\lambda_1 > 0$. We let $y := y^{(1)}$, $\theta := \theta_1$, $\lambda := \lambda_1$ and $\Tilde{y} := \frac{y^{(1)}}{\sqrt{\lambda_1}}$ to simplify the notation in the proofs. We have
\begin{align*}
    |r(t)| &= |\Tilde{y}| - |\cos\theta| \left(u_1^2(t)-v_1^2(t)\right) - |\sin\theta| \left(u_2^2(t) - v_2^2(t)\right),\\
    f_1(t) &= \lambda u_1(t) |\cos\theta r(t)|, \\
    f_2(t) &= \lambda u_2(t) |\sin\theta r(t)|, \\
    u'_1(t) &= \frac{f_1(t)}{f_1(t) + \varepsilon},~u'_2(t) := \frac{f_2(t)}{f_2(t) + \varepsilon}.
\end{align*}

\begin{lemma}\label{lemma:ratio}
    We have $u'_1(t) \geq u_2'(t)$ for $ t \in [0, T)$, and $u'_1(t) \geq \frac{2|\cot\theta|}{1+|\cot\theta|} u_2'(t)$ for $ t \in [T, \infty)$, where $T$ is the stage transition time as in Proposition~\ref{proposition:mainstage}.
\end{lemma}

\begin{proof}
    First, we show that for all $t \geq 0$, 
    \begin{equation}\label{eq:derivativeratiolb}
        u'_1(t) \geq \frac{|\cot\theta|\left(f_2(t) + \varepsilon\right)}{|\cot\theta| f_2(t)+\varepsilon} u_2'(t).
    \end{equation}
    Since $|\cos\theta| \geq |\sin\theta| > 0$ by Assumption~\ref{assumption:data2}, we have $u_1(t) \geq u_2(t)$ by Lemma~\ref{lemma:xixjcomparison}. Then,
    \begin{align*}
        f_1(t) &= \lambda u_1(t) |\cos \theta r(t)| \\
        &= |\cot\theta| \lambda u_1(t) |\sin \theta r(t)|\\
        &\geq |\cot\theta| \lambda u_2(t) |\sin \theta r(t)| \\
        &= |\cot\theta| f_2(t).
    \end{align*}
    Therefore,
    \begin{equation}
        u_1'(t) = \frac{f_1(t)}{f_1(t)+\varepsilon}
        = 1 - \frac{\varepsilon}{f_1(t) + \varepsilon}
        \geq 1 - \frac{\varepsilon}{|\cot\theta| f_2(t) + \varepsilon}
        = \frac{|\cot\theta| f_2(t)}{|\cot\theta| f_2(t) + \varepsilon}.\label{eq:u1'inequality1}
    \end{equation}
    When $u'_2(t) = 0$, \eqref{eq:derivativeratiolb} holds since $u'_1(t)$ is always non-negative. When $u'_2(t) \neq 0$, using \eqref{eq:u1'inequality1}, we also get that \eqref{eq:derivativeratiolb} holds:
    \begin{align*}
        \frac{u'_1(t)}{u'_2(t)} &= \frac{f_1(t)}{f_1(t)+\varepsilon} \cdot \frac{f_2(t) + \varepsilon}{f_2(t)} \\
        &\geq \frac{|\cot\theta| f_2(t)}{|\cot\theta| f_2(t) + \varepsilon} \cdot \frac{f_2(t) + \varepsilon}{f_2(t)} \\
        &= \frac{|\cot\theta|(f_2(t) + \varepsilon)}{|\cot\theta| f_2(t) + \varepsilon}.
    \end{align*}
    By Proposition~\ref{proposition:mainstage}, there exist stage transition times $T_1, T_2$ for $f_1(t)$ and $f_2(t)$, respectively. We know that $f_1(t) \leq \varepsilon$ for $t \geq T_1$. Since $|\cos\theta| \geq |\sin\theta|$, $f_1(t) \geq f_2(t) > \varepsilon$ for $t \in [0, T_2)$. As a result, we must have $T_1 \geq T_2$. Then, by definition, $T := \min\{T_1, T_2\} = T_2$. For all $t$, $|\cos\theta| \geq |\sin\theta|$ implies $u_1(t) \geq u_2(t)$ and $f_1(t) \geq f_2(t)$. Therefore, we conclude $u'_1(t) \geq u'_2(t)$ for $t \in [0, T)$.

    For $t \in [T, \infty)$, we have $f_2(t) \leq \varepsilon$. Notice that $\frac{|\cot\theta|(f_2 + \varepsilon)}{|\cot\theta| f_2 + \varepsilon} = 1 + \frac{(|\cot\theta|-1)\varepsilon}{|\cot\theta|f_2 + \varepsilon}$. Since $|\cot\theta| \geq 1$, the ratio $\frac{|\cot\theta|(f_2 + \varepsilon)}{|\cot\theta| f_2 + \varepsilon}$ is non-increasing in $f_2 \geq 0$. Using $f_2(t) \leq \varepsilon$, we get \begin{equation*}
        \frac{|\cot\theta|(f_2(t) + \varepsilon)}{|\cot\theta| f_2(t) + \varepsilon} \geq \frac{|\cot\theta|(\varepsilon + \varepsilon)}{|\cot\theta| \varepsilon + \varepsilon} = \frac{2|\cot\theta|}{1 + |\cot\theta|}.
    \end{equation*}
    Using \eqref{eq:derivativeratiolb}, we conclude that for $t \in [T, \infty)$, \begin{equation*}
        u'_1(t) \geq \frac{|\cot\theta|\left(f_2(t) + \varepsilon\right)}{|\cot \theta| f_2(t) + \varepsilon} u'_2(t) \geq \frac{2|\cot\theta|}{1 + |\cot\theta|} u'_2(t).
    \end{equation*}
\end{proof}

\begin{lemma}\label{lemma:M+} We let
    $\Delta := |\cos\theta| \left(u_2^\infty - u_2(0)\right) - |\sin\theta| \left(u_1^\infty - u_1(0)\right)$. We can show that $\Delta \leq M_+$, where $M_+ := \left( |\cos\theta| - |\sin\theta| \right) \left(\lambda^{-\frac{1}{4}}|y|^\frac{1}{2} - \frac{\sqrt{2}\varepsilon}{4\lambda^\frac{1}{2}|y|}\right)$.
\end{lemma}

\begin{proof}
We complete the proof in three steps.

\textbf{Step 1}. Show an upper bound for $u_2(T)$.

Let us define $p(U) := \lambda|\sin\theta|U\left(|\Tilde{y}| + \left(|\cos\theta| + |\sin\theta|\right)\alpha^2 -\left(|\cos\theta| + |\sin\theta|\right)U^2\right)$, which is a cubic function of $U \in \mathbb{R}$. Let $\hat{U}$ be the largest solution to $p(U) = \varepsilon$. Let us define $f_+(t) := \left(p \circ u_2\right)(t)$. We want to show that $f_+(t) \geq f_2(t)$ for $t \in [0, T]$.
Indeed, since $u_1(t) \geq u_2(t)$, $v_1(t), v_2(t) \leq \alpha$ always hold, we have \begin{align*}
    f_+(t) &=  \lambda |\sin\theta| u_2(t) \left( |\Tilde{y}| + \left(|\cos\theta| + |\sin\theta|\right)\alpha^2 - \left(|\cos\theta| + |\sin\theta|\right)u_2^2(t) \right) \\
    &\geq \lambda |\sin \theta| u_2(t) \left(|\Tilde{y}| + |\cos\theta| v_1^2(t) + |\sin\theta| v_2^2(t) - |\cos\theta| u_1^2(t) - |\sin\theta| u_2^2(t)\right) \\
    &= f_2(t).
\end{align*}
We know that $f_2(T) = \varepsilon$, so $f_+(t) = p(u_2(T)) \geq \varepsilon$. Meanwhile, $p(\hat{U}) = \varepsilon$. We want to show that $u_2(T) \leq \hat{U}$. Suppose for contradiction that $u_2(T) > \hat{U}$. By studying the behavior of the cubic function $p(U)$, we see that $p(U) < 0$ for sufficiently large $U$. Since $p(u_2(T)) \geq \varepsilon$, by continuity, there exists $U' \geq u_2(T)$ such that $p(U') = \varepsilon$. However, $U' \geq u_2(T) > \hat{U}$, which contradicts that $\hat{U}$ is the largest solution to $p(U) = \varepsilon$. Thus, $u_2(T) \leq \hat{U}$. By using the expansion of the cubic root $\hat{U}$ in $\varepsilon$, it is easy to show that $\hat{U} < \Tilde{u}_2 := \sqrt{\frac{|\Tilde{y}|}{|\cos\theta|+|\sin\theta|} + \alpha^2} - \frac{\varepsilon}{2\lambda|\sin\theta \Tilde{y}|}$ under Assumption~\ref{assumption:epsilonalpha}. As a result, \begin{equation}
    u_2(T) \leq \hat{U} < \Tilde{u}_2 := \sqrt{\frac{|\Tilde{y}|}{|\cos\theta|+|\sin\theta|} + \alpha^2} - \frac{\varepsilon}{2\lambda|\sin\theta \Tilde{y}|}.\label{eq:lowerboundu2T}
\end{equation}

\textbf{Step 2}. Show that $u_1(T) - u_1(0) \geq u_2(T) - u_2(0)$ and $u_1^{\infty} - u_1(T) \geq \frac{2|\cot\theta|}{1 + |\cot\theta|}\left( u_2^\infty - u_2(T) \right)$.

For all $t$, we know that $u'_1(t) \geq u_2'(t)$. Integrate both sides with respect to $t$ from $0$ to $T$, then we get \begin{equation}
    u_1(T) - u_1(0) \geq u_2(T) - u_2(0).\label{eq:stage1ratio}
\end{equation}
For $t \geq T$, by Lemma~\ref{lemma:ratio} we have $u'_1(t) \geq \frac{2|\cot\theta|}{1+|\cot\theta|}u'_2(t)$. Integrate both sides, then we get \begin{equation}
    u_1^{\infty} - u_1(T) \geq \frac{2|\cot\theta|}{1+|\cot\theta|}\left(u_2^{\infty} - u_2(T)\right).\label{eq:stage2ratio}
\end{equation}

\textbf{Step 3}. Derive an upper bound for $\Delta$.

We can write $\Delta = \Delta_1 + \Delta_2$, where
\begin{align*}
    \Delta_1 &:= |\cos\theta| (u_2(T) - u_2(0)) - |\sin\theta| (u_1(T) - u_1(0)), \\
    \Delta_2 &:= |\cos\theta| (u_2^{\infty} - u_2(T)) - |\sin \theta| (u_1^{\infty} - u_1(T)).
\end{align*}
Using \eqref{eq:stage1ratio} and \eqref{eq:stage2ratio} from Step 2, we get
\begin{align}
    \Delta_1 &\leq (|\cos\theta| - |\sin\theta|)(u_2(T)-u_2(0))\label{eq:u2T-u20}, \\
    \Delta_2 &\leq \left(|\cos\theta| - |\sin\theta| \frac{2|\cot\theta|}{1+|\cot\theta|}\right)\left( u_2^{\infty} - u_2(T) \right).\label{eq:u2inf-u2T}
\end{align}
Adding \eqref{eq:u2inf-u2T} and \eqref{eq:u2T-u20}, we get \begin{align*}
    \Delta \leq &\left(|\cos\theta| - |\sin\theta| \frac{2|\cot\theta|}{1+|\cot\theta|}\right)\left( u_2^{\infty} - u_2(T) \right) + \left(|\cos\theta| - |\sin\theta|\right)\left(u_2(T) - u_2(0) \right) \\
    = &\left(|\cos\theta| - |\sin\theta| \frac{2|\cot\theta|}{1+|\cot\theta|}\right)\left( u_2^{\infty} - \Tilde{u}_2 \right) \\
    &+ \left(|\cos\theta| - |\sin\theta| \frac{2|\cot\theta|}{1+|\cot\theta|}\right)\left( \Tilde{u}_2 - u_2(T) \right) + (|\cos\theta| - |\sin\theta|)(u_2(T) - \Tilde{u}_2) \\
    &+ (|\cos\theta| - |\sin\theta|)(\Tilde{u}_2  - u_2(0)).
\end{align*}
We have shown that $\Tilde{u}_2 \geq u_2(T)$ in \eqref{eq:lowerboundu2T}, and $|\cos\theta| \geq |\sin \theta|$ implies $|\cos\theta| - |\sin\theta| \geq |\cos\theta| - |\sin\theta| \frac{2|\cot\theta|}{1+|\cot\theta|} \geq 0$. Then,
\begin{align*}
    &\left(|\cos\theta| - |\sin\theta| \frac{2|\cot\theta|}{1+|\cot\theta|}\right)\left( \Tilde{u}_2 - u_2(T) \right) \leq \left(|\cos\theta| - |\sin\theta| \right)\left( \Tilde{u}_2 - u_2(T) \right)\\
    &\left(|\cos\theta| - |\sin\theta| \frac{2|\cot\theta|}{1+|\cot\theta|}\right)\left( \Tilde{u}_2 - u_2(T) \right) + (|\cos\theta| - |\sin\theta|)(u_2(T) - \Tilde{u}_2) \leq 0.
\end{align*}
Therefore, \begin{align}
    \Delta \leq \left(|\cos\theta| - |\sin\theta| \frac{2|\cot\theta|}{1+|\cot\theta|}\right)\left( u_2^{\infty} - \Tilde{u}_2 \right)
    + (|\cos\theta| - |\sin\theta|)(\Tilde{u}_2  - u_2(0)).\label{eq:Deltaupperbound1}
\end{align}
Moreover, by Lemma~\ref{lemma:convergence}, we know that the residual converges to zero. As a result, $\lim_{t \to \infty} r(t) = 0$, and
\begin{equation*}
    |\Tilde{y}| = |\cos\theta|\left((u_1^{\infty})^2 - (v_1^{\infty})^2\right) + |\sin\theta|\left( (u_2^{\infty})^2 - (v_2^{\infty})^2 \right).
\end{equation*}
Because $u_1(t) \geq u_2(t)$ and $v_1(t), v_2(t) \leq \alpha$ always hold, it follows that $u_2^{\infty} \leq \sqrt{\frac{|\Tilde{y}|}{|\cos\theta| + |\sin\theta|} + \alpha^2}$. Plug in $\Tilde{u}_2$ from \eqref{eq:lowerboundu2T}, then we get $u_2^\infty - \Tilde{u}_2 \leq \frac{\varepsilon}{2\lambda|\sin\theta \Tilde{y}|}$. Continue with \eqref{eq:Deltaupperbound1} and use $u_2(0) = \alpha$, then we get
\begin{align*}
    \Delta \leq &\left(|\cos\theta| - |\sin\theta| \frac{2|\cot\theta|}{1+|\cot\theta|}\right) \frac{\varepsilon}{2\lambda|\sin\theta \Tilde{y}|} \\
    + &\left( |\cos\theta| - |\sin\theta| \right)\left(\sqrt{\frac{|\Tilde{y}|}{|\cos\theta|+|\sin\theta|} + \alpha^2} - \frac{\varepsilon}{2\lambda|\sin\theta \Tilde{y}|} - u_2(0)\right) \\
    = &\left( |\cos\theta| - |\sin\theta| \right)\left(\sqrt{\frac{|\Tilde{y}|}{|\cos\theta|+|\sin\theta|} + \alpha^2} -u_2(0)\right) \\
    &+ \left( |\cos\theta| - |\sin\theta| \frac{2|\cot\theta|}{1+|\cot\theta|} - |\cos\theta| + |\sin\theta|\right)\frac{\varepsilon}{2\lambda|\sin\theta \Tilde{y}|} \\
    = &\left( |\cos\theta| - |\sin\theta| \right)\left(\sqrt{\frac{|\Tilde{y}|}{|\cos\theta|+|\sin\theta|} + \alpha^2} -\alpha\right) - \left(\frac{|\cos\theta|-|\sin\theta|}{|\cos\theta| + |\sin\theta|}\right)\frac{\varepsilon}{2\lambda|\Tilde{y}|} \\
    \leq &\left( |\cos\theta| - |\sin\theta| \right)\sqrt{\frac{|\Tilde{y}|}{|\cos\theta|+|\sin\theta|}} - \left(\frac{|\cos\theta|-|\sin\theta|}{|\cos\theta| + |\sin\theta|}\right)\frac{\varepsilon}{2\lambda|\Tilde{y}|} \\
    \leq &\left( |\cos\theta| - |\sin\theta| \right)\sqrt{|\Tilde{y}|} - \left(|\cos\theta|-|\sin\theta|\right)\frac{\sqrt{2}\varepsilon}{4\lambda|\Tilde{y}|} \\
    = &\left( |\cos\theta| - |\sin\theta| \right) \left(|y|^\frac{1}{2}\lambda^{-\frac{1}{4}} - \frac{\sqrt{2}\varepsilon}{4\lambda^\frac{1}{2}|y|}\right) \\
    = &M_+.
\end{align*}
We conclude that $\Delta \leq M_+$.
\end{proof}

\begin{lemma}\label{lemma:M-}
We let $\Delta := |\cos\theta| \left(u_2^\infty - u_2(0)\right) - |\sin\theta| \left(u_1^\infty - u_1(0)\right)$. We can show that $\Delta \geq M_-$, where $M_- := \left(|\cos\theta|-|\sin\theta|\right)\left(\left(2\lambda\right)^{-\frac{1}{4}}|y|^{\frac{1}{2}}-\alpha\right) - 2\sqrt{\frac{2\varepsilon}{\lambda^{\frac{3}{4}}|\sin\theta||y|^{\frac{1}{2}}}} - \frac{3\sqrt{2}\varepsilon}{\lambda^{\frac{1}{2}}|\sin\theta y|}\ln \left(\frac{\lambda^{\frac{1}{4}}|\sin\theta||y|^{\frac{3}{2}}}{\sqrt{2}\varepsilon}\right)$.
\end{lemma}

\begin{proof}
We begin by showing a lower bounding function for $f_2(t)$ for $t \in [0, T]$. Let $\bar{v} := |\cos\theta|\left(v_1^\infty\right)^2 + |\sin\theta|\left(v_2^\infty\right)^2$. Since $v_1(t), v_2(t)$ are non-increasing and non-negative, we have \begin{equation}
 0 \leq \bar{v} \leq |\cos\theta|v_1^2(t) + |\sin\theta|v_2^2(t) \leq \left(|\cos\theta| + |\sin\theta|\right)\alpha^2. \label{eq:barvrange}
\end{equation} Let us define \begin{equation*}
    f_-(t) := \frac{1}{2}\lambda |\sin\theta| \left(\alpha + t\right)\left(|\Tilde{y}| + \bar{v} - (|\cos\theta|+|\sin\theta|)\left(\alpha + t\right)^2\right).
\end{equation*}
For $t \leq T$, since $f_2(t) \geq \varepsilon$ and $u'_2(t) \geq \frac{1}{2}$, we have $u_2(t) \geq \alpha+\frac{1}{2}t > \frac{1}{2}\left(\alpha + t\right)$. Also, $u_1(t) \leq \alpha + t$ and $u_2(t) \leq \alpha + t$ always hold. Therefore, for $t \in [0, T]$, it follows that \begin{align*}
    f_-(t)  &= \frac{1}{2}\lambda |\sin\theta| \left(\alpha + t\right)\left(|\Tilde{y}| + \bar{v} - (|\cos\theta|+|\sin\theta|)\left(\alpha + t\right)^2\right) \\
    &< \lambda|\sin\theta|u_2(t)\left(|\Tilde{y}| + \bar{v} - |\cos\theta|u_1^2(t) - |\sin\theta|u_2^2(t)\right) \\
    &\leq \lambda|\sin\theta|u_2(t)\left(|\Tilde{y}| - |\cos\theta|\left(u_1^2(t) - v_2^2(t)\right) - |\sin\theta|\left(u_2^2(t) - v_2^2(t)\right)\right) \\
    &= f_2(t)
\end{align*}
Then, \begin{equation}
    f_-(T) < f_2(T) = \varepsilon\label{eq:f_-Tlessthanepsilon}.
\end{equation}
Assumption~\ref{assumption:epsilonalpha} guarantees $\sqrt{\frac{|\Tilde{y}|+\bar{v}}{3\left(|\cos\theta| + |\sin\theta|\right)}}-\alpha > 0$. The derivative of the cubic function $f_-(t)$ shows that $f_-(t)$ is increasing on $\left[0,~\sqrt{\frac{|\Tilde{y}|+\bar{v}}{3\left(|\cos\theta| + |\sin\theta|\right)}}-\alpha\right)$ and decreasing for $t > \sqrt{\frac{|\Tilde{y}|+\bar{v}}{3\left(|\cos\theta| + |\sin\theta|\right)}}-\alpha$. Because $f_-(0) > \varepsilon$ by Assumption~\ref{assumption:epsilonalpha} and $f_-(T) < \varepsilon$ by \eqref{eq:f_-Tlessthanepsilon}, it follows that there exists a unique $T' \in \left(\sqrt{\frac{|\Tilde{y}|+\bar{v}}{3\left(|\cos\theta| + |\sin\theta|\right)}}-\alpha,~T\right)$ such that $f_-(T') = \varepsilon$, and $f_-(t) < \varepsilon$ for $t > T'$.

Next, we show a lower bound for $\alpha + T'$. Since we already have $\alpha + T' > \sqrt{\frac{|\Tilde{y}|+\bar{v}}{3\left(|\cos\theta| + |\sin\theta|\right)}}$, then
\begin{align*}
    \varepsilon = f_-(T') &= \frac{1}{2}\lambda |\sin\theta|\left(\alpha+T'\right)\left(|\Tilde{y}| + \bar{v} -(|\cos\theta|+|\sin\theta|)\left(\alpha + T'\right)^2\right) \\
    &\geq \frac{1}{2}\lambda |\sin\theta|\sqrt{\frac{|\Tilde{y}|+\bar{v}}{3\left(|\cos\theta| + |\sin\theta|\right)}}\left(|\Tilde{y}|+\bar{v}-(|\cos\theta|+|\sin\theta|)\left(\alpha + T'\right)^2\right).
\end{align*}
Therefore, \begin{align}
    \frac{2\varepsilon}{\lambda|\sin\theta|}\sqrt{\frac{3(|\cos\theta|+|\sin\theta|)}{|\Tilde{y}|+\bar{v}}} &\geq |\Tilde{y}|+\bar{v}-(|\cos\theta|+|\sin\theta|)\left(\alpha + T'\right)^2 \\
    (|\cos\theta|+|\sin\theta|)\left(\alpha + T'\right)^2 &\geq |\Tilde{y}|+\bar{v} - \frac{2\varepsilon}{\lambda|\sin\theta|}\sqrt{\frac{3(|\cos\theta|+|\sin\theta|)}{|\Tilde{y}|+\bar{v}}} \\
    (\alpha + T')^2 &\geq \frac{|\Tilde{y}|+\bar{v}}{|\cos\theta|+|\sin\theta|} - \frac{2\varepsilon}{\lambda|\sin\theta|}\sqrt{\frac{3}{(|\Tilde{y}|+\bar{v})(|\cos\theta|+|\sin\theta|)}} \\
    (\alpha + T')^2 &\geq \frac{|\Tilde{y}|+\bar{v}}{|\cos\theta|+|\sin\theta|} - \frac{2\sqrt{3}\varepsilon}{\lambda|\sin\theta|(|\Tilde{y}|+\bar{v})^{\frac{1}{2}}} \\
    (\alpha + T')^2 &\geq \frac{|\Tilde{y}|+\bar{v}}{|\cos\theta|+|\sin\theta|} - \frac{4\varepsilon}{\lambda|\sin\theta|(|\Tilde{y}|+\bar{v})^{\frac{1}{2}}} \\
    \alpha + T' &\geq \left(\frac{|\Tilde{y}|+\bar{v}}{|\cos\theta|+|\sin\theta|} - \frac{4\varepsilon}{\lambda|\sin\theta|(|\Tilde{y}|+\bar{v})^{\frac{1}{2}}}\right)^{\frac{1}{2}}  \\
    \alpha + T' &\geq \sqrt{\frac{|\Tilde{y}|+\bar{v}}{|\cos\theta|+|\sin\theta|}} - 2\sqrt{\frac{\varepsilon}{\lambda|\sin\theta|(|\Tilde{y}|+\bar{v})^{\frac{1}{2}}}}\label{eq:lowerboundoft1-final}
\end{align}

Then, we want to find a lower bound for $u_2(T)$. Because $u_2(t)$ is non-decreasing, $T > T'$ implies $u_2(T) \geq u_2(T')$. For all $t \in [0, T']$, $f_-(t) \leq f_2(t)$ holds, then \begin{align*}
    u'_2(t) &= \frac{f_2(t)}{f_2(t) + \varepsilon} \\
    &\geq \frac{f_-(t)}{f_-(t) + \varepsilon} \\
    &= 1 - \frac{2\varepsilon}{\lambda|\sin\theta|\left(\alpha + t\right)\left(|\Tilde{y}| +\bar{v}- \left( |\cos\theta| + |\sin\theta|\right)\left(\alpha + t\right)^2\right)}.
 \end{align*}
 This lower bounding function is explicit in $t$, which makes it possible to obtain a lower bound for $u_2(T')$ by integrating it with respect to $t$ from $0$ to $T'$.\begin{align}
     u_2(T') - u_2(0) &\geq \int_{0}^{T'} 1 - \frac{2\varepsilon}{\lambda|\sin\theta|\left(\alpha + t\right)\left(|\Tilde{y}| +\bar{v}- \left( |\cos\theta| + |\sin\theta|\right)\left(\alpha + t\right)^2\right)} \, dt \\
     u_2(T') &\geq \alpha + T' - \frac{2\varepsilon}{\lambda|\sin\theta|}\int_{0}^{T'} \frac{1}{\left(\alpha + t\right)\left(|\Tilde{y}| +\bar{v}- \left( |\cos\theta| + |\sin\theta|\right)\left(\alpha + t\right)^2\right)} \, dt \label{eq:u2T'lowerbound}
 \end{align}
 Let $\tau := \alpha + t$, and compute the integral\begin{align*}
    J:= \int_\alpha^{\alpha+T'} \frac{1}{\tau(|\Tilde{y}| +\bar{v}- (|\cos\theta|+|\sin\theta|)\tau^2)}\, d\tau &= \left.\frac{1}{2(|\Tilde{y}|+\bar{v})}\ln{\frac{\tau^2}{|\Tilde{y}| + \bar{v} - (|\cos\theta|+|\sin\theta|)\tau^2}}\right|_{\alpha}^{\alpha+T'} \\
    &= \frac{1}{2(|\Tilde{y}|+\bar{v})}\ln{\frac{(\alpha+T')^2(|\Tilde{y}|+\bar{v} - (|\cos\theta|+|\sin\theta|)\alpha^2)}{\alpha^2(|\Tilde{y}|+\bar{v} - (|\cos\theta|+|\sin\theta|)(\alpha+T')^2)}}.
 \end{align*}
Since $f_-(T') = \varepsilon$, we get \begin{equation*}
    |\Tilde{y}| + \bar{v} - (|\cos\theta|+|\sin\theta|)\left(\alpha + T'\right)^2 = \frac{2\varepsilon}{\lambda|\sin\theta|(\alpha+ T')}.
\end{equation*}
Also, $\left(\alpha + T'\right)^2 \leq |\Tilde{y}| + \bar{v}$. Using $\bar{v} \leq \left( |\cos\theta| + |\sin\theta|\right)\alpha^2$ by \eqref{eq:barvrange} and Assumption~\ref{assumption:epsilonalpha}, we get 
\begin{align*}
    J &\leq \frac{1}{2|(\Tilde{y}|+\bar{v})}\ln \frac{\left(\alpha + T' \right)^3}{2\varepsilon/(\lambda|\sin\theta|)} \frac{|\Tilde{y}|}{\alpha^2} \\
    &\leq \frac{1}{2(|\Tilde{y}|+\bar{v})}\ln \frac{(|\Tilde{y}|+\bar{v})^\frac{3}{2}}{2\varepsilon/(\lambda|\sin\theta|)} \frac{|\Tilde{y}|}{\left(2\varepsilon/(\lambda|\sin\theta\Tilde{y}|) \right)^2} \\
    &= \frac{1}{2(|\Tilde{y}|+\bar{v})}\ln \left((|\Tilde{y}|+\bar{v})^\frac{3}{2}\left(\frac{\lambda|\sin\theta\Tilde{y}|}{2\varepsilon}\right)^3\right) \\
    &= \frac{3}{2(|\Tilde{y}|+\bar{v})}\ln \left((|\Tilde{y}|+\bar{v})^\frac{1}{2}\left(\frac{\lambda|\sin\theta\Tilde{y}|}{2\varepsilon}\right)\right).
\end{align*}
Using Assumption~\ref{assumption:epsilonalpha}, we get \begin{align*}
    |\Tilde{y}| + \bar{v} \leq |\Tilde{y}| + \left(|\cos\theta|+|\sin\theta|\right)\alpha^2 \leq |\Tilde{y}| + \frac{1}{18}|\Tilde{y}| = \frac{19}{18}|\Tilde{y}|.
\end{align*}
Then, \begin{align*}
    J &\leq \frac{3}{2(|\Tilde{y}|+\bar{v})}\ln \left((|\Tilde{y}|+\bar{v})^\frac{1}{2}\left(\frac{\lambda|\sin\theta\Tilde{y}|}{2\varepsilon}\right)\right) \\
    &\leq \frac{3}{2(|\Tilde{y}|+\bar{v})} \ln \left(\left(\frac{19}{18}|\Tilde{y}|\right)^\frac{1}{2}\left(\frac{\lambda|\sin\theta\Tilde{y}|}{2\varepsilon}\right)\right) \\
    &\leq \frac{3}{2(|\Tilde{y}|+\bar{v})} \ln \left(\sqrt{2}|\Tilde{y}|^\frac{1}{2}\left(\frac{\lambda|\sin\theta\Tilde{y}|}{2\varepsilon}\right)\right) \\
    &= \frac{3}{2(|\Tilde{y}|+\bar{v})} \ln \left(\frac{\lambda|\sin\theta||\Tilde{y}|^\frac{3}{2}}{\sqrt{2}\varepsilon}\right).
\end{align*}
Assumption~\ref{assumption:epsilonalpha} implies that $\varepsilon \leq \frac{\lambda|\sin\theta||\Tilde{y}|^{\frac{3}{2}}}{9\sqrt{2(|\cos\theta|+|\sin\theta|)}}$. Therefore, $\ln \left(\frac{\lambda|\sin\theta||\Tilde{y}|^\frac{3}{2}}{\sqrt{2}\varepsilon}\right)$ is guaranteed to be positive, and then we get\begin{equation}
    J \leq \frac{3}{2|\Tilde{y}|} \ln \left(\frac{\lambda|\sin\theta||\Tilde{y}|^\frac{3}{2}}{\sqrt{2}\varepsilon}\right). \label{eq:integral}
\end{equation}

Combining \eqref{eq:lowerboundoft1-final}, \eqref{eq:u2T'lowerbound} and \eqref{eq:integral}, we get \begin{align*}
u_2(T') &\geq \sqrt{\frac{|\Tilde{y}|+\bar{v}}{|\cos\theta|+|\sin\theta|}} - 2\sqrt{\frac{\varepsilon}{\lambda|\sin\theta|(|\Tilde{y}|+\bar{v})^{\frac{1}{2}}}} - \frac{3\varepsilon}{\lambda|\sin\theta\Tilde{y}|}\ln \left(\frac{\lambda|\sin\theta||\Tilde{y}|^\frac{3}{2}}{\sqrt{2}\varepsilon}\right) \\
&\geq \sqrt{\frac{|\Tilde{y}|+\bar{v}}{|\cos\theta|+|\sin\theta|}} - 2\sqrt{\frac{\varepsilon}{\lambda|\sin\theta||\Tilde{y}|^{\frac{1}{2}}}} - \frac{3\varepsilon}{\lambda|\sin\theta\Tilde{y}|}\ln \left(\frac{\lambda|\sin\theta||\Tilde{y}|^\frac{3}{2}}{\sqrt{2}\varepsilon}\right).
\end{align*}
Let $P := \sqrt{\frac{|\Tilde{y}|+\bar{v}}{|\cos\theta|+|\sin\theta|}}$, $Q := 2\sqrt{\frac{\varepsilon}{\lambda|\sin\theta||\Tilde{y}|^{\frac{1}{2}}}} + \frac{3\varepsilon}{\lambda|\sin\theta \Tilde{y}|}\ln \left(\frac{\lambda|\sin\theta ||\Tilde{y}|^{\frac{3}{2}}}{\sqrt{2}\varepsilon}\right)$, and $P, Q > 0$. Then, $u_2(T') \geq P - Q$. Using Lemma~\ref{lemma:convergence}, we get \begin{align*}
    |\cos\theta|\left(P + Q\right)^2 + |\sin\theta|\left( P - Q \right)^2  &= (|\cos\theta|+|\sin\theta|)P^2 + 2PQ(|\cos\theta|-|\sin\theta|) + (|\cos\theta|+|\sin\theta|)Q^2 \\
    &\geq (|\cos\theta|+|\sin\theta|)P^2 \\
    &= |\Tilde{y}| + \bar{v} \\
    &= |\Tilde{y}| + |\cos\theta|\left(v_1^{\infty}\right)^2 + |\sin\theta|\left(v_2^\infty\right)^2 \\
    &= |\cos\theta|\left(u_1^{\infty}\right)^2 + |\sin\theta|\left(u_2^{\infty}\right)^2.
\end{align*}
Since $u_2(t)$ is non-decreasing, we get $u_2^\infty \geq u_2(T') \geq P - Q$. As a result, we must have $ u_1^\infty \leq P + Q $. Therefore,
\begin{align}
    |\cos\theta|u_2^\infty - |\sin\theta| u_1^\infty
    &\geq \left(|\cos\theta|-|\sin\theta|\right)P - \left(|\cos\theta| + |\sin\theta|\right) Q \\
    &= \left(|\cos\theta|-|\sin\theta|\right)\sqrt{\frac{|\Tilde{y}|+\bar{v}}{|\cos\theta| + |\sin\theta|}} - \left(|\cos\theta| + |\sin\theta|\right) Q \\
    &\geq \left(|\cos\theta|-|\sin\theta|\right)\sqrt{\frac{|\Tilde{y}|}{|\cos\theta| + |\sin\theta|}} - \sqrt{2}Q \\
    &\geq \left(|\cos\theta|-|\sin\theta|\right)\sqrt{\frac{|\Tilde{y}|}{\sqrt{2}}} - \sqrt{2}Q.\label{eq:finaleq1}
\end{align}
Additionally, we have $u_1(0) = u_2(0) = \alpha$, then \begin{equation}
    -|\cos\theta|u_2(0) + |\sin\theta|u_1(0) = -\alpha (|\cos\theta|-|\sin\theta|). \label{eq:finaleq2}
\end{equation}
Adding \eqref{eq:finaleq1} and \eqref{eq:finaleq2}, and plugging in $Q$, we get 
\begin{align*}
    \Delta &= |\cos\theta|(u_2^\infty - u_2(0)) - |\sin\theta|(u_1^\infty - u_1(0)) \\
    &= |\cos\theta|u_2^\infty - |\sin\theta| u_1^\infty -\alpha(|\cos\theta| - |\sin\theta|) \\
    &\geq \left(|\cos\theta|-|\sin\theta|\right)\left(\sqrt{\frac{|\Tilde{y}|}{\sqrt{2}}}-\alpha\right) - 2\sqrt{\frac{2\varepsilon}{\lambda|\sin\theta||\Tilde{y}|^{\frac{1}{2}}}} - \frac{3\sqrt{2}\varepsilon}{\lambda|\sin\theta \Tilde{y}|}\ln \left(\frac{\lambda|\sin\theta|| \Tilde{y}|^{\frac{3}{2}}}{\sqrt{2}\varepsilon}\right). \end{align*}
Finally, plugging in $\Tilde{y} = \lambda^{-\frac{1}{2}}y$, we get
\begin{align*}
    \Delta &\geq \left(|\cos\theta|-|\sin\theta|\right)\left(\left(2\lambda\right)^{-\frac{1}{4}}|y|^{\frac{1}{2}}-\alpha\right) - 2\sqrt{\frac{2\varepsilon}{\lambda^{\frac{3}{4}}|\sin\theta||y|^{\frac{1}{2}}}} - \frac{3\sqrt{2}\varepsilon}{\lambda^{\frac{1}{2}}|\sin\theta y|}\ln \left(\frac{\lambda^{\frac{1}{4}}|\sin\theta||y|^{\frac{3}{2}}}{\sqrt{2}\varepsilon}\right)\\
    &= M_-.
\end{align*}
Therefore, $\Delta \geq M_-$.
\end{proof}

\begin{lemma}\label{lemma:M-_decreasing}
Let us consider $M_-(\varepsilon)$ as a function of $\varepsilon$, where $M_-(0) := \lim_{\varepsilon\to0^+} M_-(\varepsilon)$. Then, it follows that $M_-(0) > 0$ and $M_-(\varepsilon)$ is strictly decreasing in $\varepsilon$ for $0 \leq \varepsilon \leq \frac{1}{9}\frac{\lambda|\sin\theta||\Tilde{y}|^{\frac{3}{2}}}{\sqrt{2(|\cos\theta|+|\sin\theta|)}}$.
\end{lemma}

\begin{proof}
    We can write $M_-(\varepsilon) = N_0 + N_1(\varepsilon) + N_2(\varepsilon)$, where \begin{align*}
        N_0 &= \left(|\cos\theta|-|\sin\theta|\right)\left(2^{-\frac{1}{4}}\sqrt{|\Tilde{y}|}-\alpha\right), \\
        N_1(\varepsilon) &= - 2\sqrt{\frac{2\varepsilon}{\lambda|\sin\theta||\Tilde{y}|^{\frac{1}{2}}}},\\
        N_2(\varepsilon) &= - \frac{3\sqrt{2}\varepsilon}{\lambda|\sin\theta \Tilde{y}|}\ln \left({|\Tilde{y}|^\frac{1}{2}}\frac{\lambda|\sin\theta \Tilde{y}|}{\sqrt{2}\varepsilon}\right).
    \end{align*} Notice that $N_0$ does not depend on $\varepsilon$, and $N_1(\varepsilon)$ is decreasing in $\varepsilon$ for all $\varepsilon \geq 0$. Let $\varepsilon' := \frac{\sqrt{2}\varepsilon}{\lambda|\sin\theta\Tilde{y}|}$, then \begin{align*}
        N_2(\varepsilon') &= -3\varepsilon' \ln \left(\frac{|\Tilde{y}|^{\frac{1}{2}}}{\varepsilon'}\right), \\
        \frac{dN_2(\varepsilon')}{d\varepsilon'} &= -3\left(\ln \left(\frac{|\Tilde{y}|^{\frac{1}{2}}}{\varepsilon'}\right) - 1\right).
    \end{align*}
    Since $0 \leq \varepsilon \leq \frac{1}{9}\frac{\lambda|\sin\theta||\Tilde{y}|^{\frac{3}{2}}}{\sqrt{2(|\cos\theta|+|\sin\theta|)}}$, then $0 \leq \varepsilon' \leq \frac{1}{9}\frac{|\Tilde{y}|^{\frac{1}{2}}}{\sqrt{|\cos\theta|+|\sin\theta|}}$. Then, \begin{align*}
        \frac{|\Tilde{y}|^{\frac{1}{2}}}{\varepsilon'} \geq 9\sqrt{|\cos\theta|+|\sin\theta|} \geq 9 > e.
    \end{align*}
    Therefore, $\ln \left(\frac{|\Tilde{y}|^{\frac{1}{2}}}{\varepsilon'}\right) > 1$ and $\frac{dN_2(\varepsilon')}{d\varepsilon'} < 0$. It follows that $N_2(\varepsilon)$ is decreasing in $\varepsilon$ on the given interval. Combining $N_0, N_1$ and $N_2$, we conclude that $M_-(\varepsilon)$ is decreasing in $\varepsilon$ on the given interval.     %
    We also have $\lim_{\varepsilon\to 0^+}N_2(\varepsilon)= 0$ by the L'Hopital's rule. Then, using Assumption~\ref{assumption:epsilonalpha}, we get
    \begin{equation*}
        M_-(0) = \left(|\cos\theta|-|\sin\theta|\right)\left(2^{-\frac{1}{4}}\sqrt{|\Tilde{y}|}-\alpha\right) > 0.
    \end{equation*}
\end{proof}

\subsection{Proof of Theorem~\ref{maintheorem}}
\begin{proof}
By Lemma~\ref{lemma:convergence}, we know that the weights converge and the residual $r(t)$ converges to zero. It follows that $\bm \beta^\infty := \lim_{t \to \infty} \bm \beta (t)$ exists and is finite. Using \eqref{eq:expressionofrtinbeta}, we get \begin{align*}
    0 = \lim_{t \to \infty}r(t) = y - \left(\sqrt{\lambda}\cos\theta \beta_1^\infty + \sqrt{\lambda}\sin\theta\beta_2^\infty \right) = \bm y - X\bm \beta^\infty.
\end{align*}
Therefore, the convergent solution $\bm\beta^\infty$ is an interpolating solution.

Next, we derive the stationary condition for the optimization problem \eqref{eq:Bregmanminimization} and get:
\begin{align}
    \nabla_{\bm \beta}\left(X\bm\beta^\infty\right) &= \begin{bmatrix}
        \sqrt{\lambda}\cos\theta \\
        \sqrt{\lambda}\sin\theta
    \end{bmatrix}, \\
    \nabla_{\bm \beta} E\left(\bm\beta^\infty\right) &= \nabla \Phi_{\infty}\left(\bm\beta^\infty \right) - \nabla \Phi_{0}\left(\bm\beta(0)\right). \label{eq:Bregmangradient}
\end{align}
The gradient in the left hand side of \eqref{eq:Bregmangradient} is equal to the difference between the convergent point and the starting point of the dual variable. We use the result of the dual dynamics in Proposition~\ref{proposition:dualdynamics} to calculate the gradient. Recall that the dual dynamics follow\begin{align*}
    \frac{d\nabla\Phi_t(\bm\beta(t))}{dt} = -\operatorname{sgn}(\bm\beta(t)) \odot \frac{\nabla_{\bm u}L(\bm w(t))}{|\nabla_{\bm u}L(\bm w(t))| + \varepsilon}.
\end{align*}
From \eqref{eq:signofbeta} we see that $\operatorname{sgn}(\beta_i(t)) = \operatorname{sgn}(x_i y)$. As a result, $\operatorname{sgn}(\bm \beta(t))$ remains the same for all $t$. Integrating both sides with respect to $t$ from $0$ to infinity, we get \begin{align*} 
   \nabla \Phi_{\infty}(\bm\beta^\infty) - \nabla \Phi_{0}(\bm\beta(0)) = \begin{bmatrix}
       \operatorname{sgn}(\cos\theta \Tilde{y}) \\
       \operatorname{sgn}(\sin\theta \Tilde{y})
   \end{bmatrix} \odot \left(\bm u^{\infty} - \bm u(0)\right).
\end{align*}
Next, we want to compute the extent of the deviation from the exact KKT point. \begin{align*}
    \delta := \min_{\nu \in \mathbb{R}}\left\Vert\nabla_{\bm \beta} E\left(\bm\beta\right) - \nu \cdot \nabla_{\bm \beta}(X\bm\beta) \right\Vert &= \min_{\nu \in \mathbb{R}}\left\Vert\begin{bmatrix}
       \operatorname{sgn}(\cos\theta \Tilde{y}) \\
       \operatorname{sgn}(\sin\theta \Tilde{y})
   \end{bmatrix}\odot \left(\bm u^\infty - \bm u(0) \right) - \nu \cdot \begin{bmatrix}
        \sqrt{\lambda}\cos\theta \\
        \sqrt{\lambda}\sin\theta
    \end{bmatrix} \right\Vert
\end{align*}
Let $V:= \begin{bmatrix}
       \operatorname{sgn}(\cos\theta\Tilde{y})\left(u_1^\infty - u_1(0) \right) \\
       \operatorname{sgn}(\sin\theta\Tilde{y})\left(u_2^\infty - u_2(0) \right)
   \end{bmatrix}$. Using orthogonal projection, we get: \begin{align*}
     \min_{\nu \in \mathbb{R}}\left\Vert V - \nu \cdot \begin{bmatrix}
        \sqrt{\lambda}\cos\theta \\
        \sqrt{\lambda}\sin\theta
    \end{bmatrix} \right\Vert &= \left| \left\langle V, ~\begin{bmatrix}
    -\sin\theta \\
    \cos\theta
\end{bmatrix} \right\rangle\right| \\
&= \Big| -\operatorname{sgn}(\cos\theta\Tilde{y})\sin\theta \left(u_1^\infty - u_1(0)\right) + \operatorname{sgn}(\sin\theta\Tilde{y})\cos\theta \left(u_2^\infty - u_2(0)\right)\Big| \\
&= \Big|\operatorname{sgn}(\sin\theta \cos\theta \Tilde{y})\cdot \left(-|\sin\theta|\left(u_1^\infty - u_1(0)\right) + |\cos\theta|\left(u_2^\infty - u_2(0)\right) \right) \Big| \\
&= \Big||\cos\theta|\left(u_2^\infty - u_2(0)\right) - |\sin\theta|\left(u_1^\infty - u_1(0)\right) \Big|.
\end{align*}
Therefore, $\delta = |\Delta|$, where $\Delta := |\cos\theta| \left(u_2^\infty - u_2(0) \right) - |\sin\theta| \left(u_1^\infty - u_1(0)\right)$. Apply Lemma~\ref{lemma:M+} and Lemma~\ref{lemma:M-}, we get\begin{align*}
    M_- \leq \Delta \leq M_+,
\end{align*}
where\begin{align*}
    M_- &:=  \left(|\cos\theta|-|\sin\theta|\right)\left(\left(2\lambda\right)^{-\frac{1}{4}}|y|^{\frac{1}{2}}-\alpha\right) - 2\sqrt{\frac{2\varepsilon}{\lambda^{\frac{3}{4}}|\sin\theta||y|^{\frac{1}{2}}}} - \frac{3\sqrt{2}\varepsilon}{\lambda^{\frac{1}{2}}|\sin\theta y|}\ln \left(\frac{\lambda^{\frac{1}{4}}|\sin\theta||y|^{\frac{3}{2}}}{\sqrt{2}\varepsilon}\right), \\
    M_+ &:= \left( |\cos\theta| - |\sin\theta| \right) \left(\lambda^{-\frac{1}{4}}|y|^\frac{1}{2} - \frac{\sqrt{2}\varepsilon}{4\lambda^\frac{1}{2}|y|}\right).
\end{align*}
Thus, $\delta = |\Delta| \leq \max\{|M_-|, |M_+|\}$.
\end{proof}

\subsection{Proof of Corollary~\ref{corollary:epsilon}}
\begin{proof}
Let us consider $\delta(\varepsilon), \Delta(\varepsilon), M_-(\varepsilon), M_+(\varepsilon)$ as functions of $\varepsilon$ on the domain $\mathcal{I}_{\varepsilon} = \left[0,~\bar{\varepsilon}\right]$ implied by Assumption~\ref{assumption:epsilonalpha}. And we define $M_-(0) := \lim_{\varepsilon\to0^+} M_-(\varepsilon)$ so that $M_-(\varepsilon)$ is continuous on the domain. Theorem~\ref{maintheorem} shows that $M_+(\varepsilon)$ is linearly decreasing in $\varepsilon$. By Lemma~\ref{lemma:M-_decreasing}, $M_-(\varepsilon)$ is strictly decreasing in $\varepsilon$ on the domain $\mathcal{I}_\varepsilon$. Lemma~\ref{lemma:M-_decreasing} also shows that $M_-(0) > 0$. If $M_-(\bar{\varepsilon}) \geq 0$, then $\Delta(\varepsilon) \geq M_-(\varepsilon) \geq M_-(\bar{\varepsilon}) \geq 0$ for all $\varepsilon \in \mathcal{I}_{\varepsilon}$. It implies that only $M_+(\varepsilon)$ applies to the bound, i.e., $\delta(\varepsilon) \leq M_+(\varepsilon)$. Let $\varepsilon^* = \bar{\varepsilon}$, and then for $\varepsilon \in [0,\,\varepsilon^*]$ we have \begin{equation}
    \delta(\varepsilon) \leq \bar{M} - \left( |\cos\theta| - |\sin\theta| \right)\frac{\sqrt{2}\varepsilon}{4\lambda^\frac{1}{2}|y|}.\label{eq:deltaepsilon}
\end{equation}
If $M_-(\bar{\varepsilon}) < 0$, since $M_-(0) > 0$, the monotonicity of $M_-(\varepsilon)$ ensures a unique $\hat{\varepsilon} \in (0, \bar{\varepsilon})$ such that $M_-(\hat{\varepsilon}) = 0$ and $\Delta(\varepsilon) \geq M_-(\varepsilon) \geq 0$ for $\varepsilon \in [0,\,\hat{\varepsilon}]$. Let $\varepsilon^* = \hat{\varepsilon}$. Notice that $\hat{\varepsilon}$ is positive, so $[0,\,\varepsilon^*]$ is non-degenerate. By similar arguments, we can show \eqref{eq:deltaepsilon}. We complete the proof by setting $\mathcal{I}' := [0,\,\varepsilon^*] \subseteq \mathcal{I}_\varepsilon$.
\end{proof}

\subsection{Proof of Corollary~\ref{corollary:NDcorollary}}
\begin{proof}
    By Lemma~\ref{lemma:convergence}, $\lim_{t\to\infty}r^{(n)}(t) = 0$ for all $n \in \{1, \dots, N\}$ and the weights converge. We let $\bar{\bm\beta}^{(n)} := \lim_{t\to\infty}\bm\beta^{(n)}(t)$, $\bar{\bm u}^{(n)} := \lim_{t\to\infty}\bm u^{(n)}(t)$ and $\bar{\bm v}^{(n)} := \lim_{t\to\infty}\bm v^{(n)}(t)$ for each $n$. We let $\bm\beta^{\infty} := \lim_{t\to\infty}\bm\beta(t) = \begin{bmatrix}
        \bar{\bm\beta}^{(1)} & \dots & \bar{\bm\beta}^{(n)}
    \end{bmatrix}^\top$.
    Then, using \eqref{eq:expressionofrtinbeta}, for all $n$ we have
    \begin{equation*}
        0 = \lim_{t\to\infty} r^{(n)}(t) = y^{(n)} - x^{(n)}_1\bar\beta_1^{(n)} - x^{(n)}_2\bar\beta_2^{(n)}.
    \end{equation*}
    Therefore, $X\bm\beta^\infty = \bm y$, i.e., $\bm\beta^\infty$ is an interpolating solution.
    
    Each block of $X^\top X$ can be parameterized by $\theta_n$ and $\lambda_n$:
    \begin{align*}
        B^{(n)} = \begin{bmatrix}
        \cos\theta_n & -\sin\theta_n \\
        \sin\theta_n & \cos\theta_n 
    \end{bmatrix}\begin{bmatrix}
        \lambda_n & 0 \\
        0 & 0
    \end{bmatrix}\begin{bmatrix}
        \cos\theta_n & \sin\theta_n \\
        -\sin\theta_n & \cos\theta_n 
    \end{bmatrix},
    \end{align*}
    where $|\cos\theta_n| \geq |\sin\theta_n| > 0$. $B^{(n)}$ is positive semi-definite and has rank 1, so $\lambda_n > 0$. We let $\Tilde{y}^{(n)} := \frac{y^{(n)}}{\sqrt{\lambda_n}}$. The constraint $X\bm\beta^\infty = \bm y$ consists of $N$ equality conditions: \begin{align*}
        \left\langle\bm x^{(1)}, \,\bm\beta^\infty\right\rangle &= y^{(1)}, \\
        &\dots \\
        \left\langle\bm x^{(N)}, \,\bm\beta^\infty\right\rangle &= y^{(N)}.
    \end{align*}
By integrating both sides of \eqref{eq:dualdynamics}, we get\begin{align*}
    \nabla_{\bm\beta} E\left(\bm\beta^\infty\right) &= \nabla\Phi_{\infty}\left(\bm\beta^\infty\right) -\nabla\Phi_{0}\left(\bm\beta(0)\right) \\
    &= \begin{bmatrix}
        \operatorname{sgn}(\cos\theta_1 \Tilde{y}^{(1)})\left( \bar{u}_1^{(1)} - u_1^{(1)}(0) \right) \\
        \operatorname{sgn}(\sin\theta_1 \Tilde{y}^{(1)})\left( \bar{u}_2^{(1)} - u_2^{(1)}(0) \right) \\
        \dots \\
        \operatorname{sgn}(\cos\theta_N \Tilde{y}^{(N)})\left( \bar{u}_1^{(N)} - u_1^{(N)}(0) \right) \\
        \operatorname{sgn}(\sin\theta_N \Tilde{y}^{(N)})\left( \bar{u}_2^{(N)} - u_2^{(N)}(0) \right) \\
    \end{bmatrix}.
\end{align*}
We let $\bm \mu := \begin{bmatrix}
    \mu_1 & \dots & \mu_N 
\end{bmatrix}$. Then, 
\begin{align*}
    \bar{\delta} &:= \min_{\bm\mu \in \mathbb{R}^N}\left\Vert \nabla_{\bm\beta} E\left(\bm\beta^\infty\right) - \sum_{n=1}^N\mu_n \bm x^{(n)} \right\Vert \\
    &= \min_{\bm\mu \in \mathbb{R}^N}\left\Vert  \nabla_{\bm\beta} E\left(\bm\beta^\infty\right) - \begin{bmatrix}
        \mu_1 \sqrt{\lambda_1}\cos\theta_1 \\
        \mu_1 \sqrt{\lambda_1}\sin\theta_1 \\
        \dots \\
        \mu_N \sqrt{\lambda_N}\cos\theta_N \\
        \mu_N \sqrt{\lambda_N}\sin\theta_N
    \end{bmatrix} \right\Vert \\
    &= \min_{\bm\mu \in \mathbb{R}^N}\left(\sum_{n=1}^{N} \left\Vert \begin{bmatrix}
       \operatorname{sgn}(\cos\theta_n \Tilde{y}^{(n)})\left( \bar{u}_1^{(n)} - u_1^{(n)}(0) \right) \\
        \operatorname{sgn}(\sin\theta_n \Tilde{y}^{(n)})\left( \bar{u}_2^{(n)} - u_2^{(n)}(0) \right) 
    \end{bmatrix} - \mu_n \cdot \begin{bmatrix}
         \sqrt{\lambda_n}\cos\theta_n \\
        \sqrt{\lambda_n}\sin\theta_n \\
    \end{bmatrix}\right\Vert^2\right)^{\frac{1}{2}} \\
    &\leq \sum_{n=1}^{N} \min_{\mu_n \in \mathbb{R}}\left\Vert \begin{bmatrix}
       \operatorname{sgn}(\cos\theta_n \Tilde{y}^{(n)})\left( \bar{u}_1^{(n)} - u_1^{(n)}(0) \right) \\
        \operatorname{sgn}(\sin\theta_n \Tilde{y}^{(n)})\left( \bar{u}_2^{(n)} - u_2^{(n)}(0) \right) 
    \end{bmatrix} - \mu_n \cdot \begin{bmatrix}
         \sqrt{\lambda_n}\cos\theta_n \\
        \sqrt{\lambda_n}\sin\theta_n \\
    \end{bmatrix}\right\Vert.
\end{align*}
By Theorem~\ref{maintheorem}, for each $n$ we have \begin{align*}
    \delta_n &:= \min_{\mu_n \in \mathbb{R}}\left\Vert \begin{bmatrix}
       \operatorname{sgn}(\cos\theta_n \Tilde{y}^{(n)})\left( \bar{u}_1^{(n)} - u_1^{(n)}(0) \right) \\
        \operatorname{sgn}(\sin\theta_n \Tilde{y}^{(n)})\left( \bar{u}_2^{(n)} - u_2^{(n)}(0) \right) 
    \end{bmatrix} - \mu_n \cdot \begin{bmatrix}
         \sqrt{\lambda_n}\cos\theta_n \\
        \sqrt{\lambda_n}\sin\theta_n \\
    \end{bmatrix}\right\Vert \\
    &\leq \max\left\{\Big|M^{(n)}_+\Big|,~\Big|M_-^{(n)}\Big|\right\}.
\end{align*}
We conclude that $\bar{\delta} \leq \sum_{n=1}^N \delta_n \leq \sum_{n=1}^N \max\left\{\Big|M^{(n)}_+\Big|,~\Big|M_-^{(n)}\Big|\right\}$.
\end{proof}

\subsection{Proof of Corollary~\ref{corollary:NDvarepsilon}}
\begin{proof}
    For each $n \in \{1, \dots, N \}$, we can apply Corollary~\ref{corollary:epsilon} and show that there exists a non-degenerate interval $\mathcal{I}'_n = [0,~\varepsilon^*_n]$ such that for all $\varepsilon \in \mathcal{I}'_n$, we have \begin{align}
    \delta_n(\varepsilon) \leq \left( |\cos\theta_n| - |\sin\theta_n| \right) \left(\lambda_n^{-\frac{1}{4}}|y^{(n)}|^\frac{1}{2}\right) - \left( |\cos\theta_n| - |\sin\theta_n| \right)\frac{\sqrt{2}\varepsilon}{4\lambda_n^\frac{1}{2}|y^{(n)}|}.\label{eq:deltanupperbound}
\end{align}
We let $\Tilde{\varepsilon} := \min_n\{ \varepsilon^*_n\}$ and let $\mathcal{J} := \bigcap_{n=1}^N \mathcal{I}'_n = [0, ~\Tilde{\varepsilon}]$. Since each $\mathcal{I}'_n$ is non-degenerate, we have $\varepsilon_n^* > 0$ for all $n$ and $\Tilde{\varepsilon} > 0$. Therefore, the interval $\mathcal{J}$ is non-degenerate. Then, for all $\varepsilon \in \mathcal{J}$, the relationship \eqref{eq:deltanupperbound} holds. Therefore, by Corollary~\ref{corollary:NDcorollary}, we have
\begin{equation*}
    \bar{\delta}(\varepsilon) \leq \sum_{n=1}^{N} \delta_n(\varepsilon) \leq \sum_{n=1}^N \left( |\cos\theta_n| - |\sin\theta_n| \right) \left(\lambda_n^{-\frac{1}{4}}|y^{(n)}|^\frac{1}{2}\right)  - \left(\sum_{n=1}^N\left( |\cos\theta_n| - |\sin\theta_n| \right)\frac{\sqrt{2}}{4\lambda_n^\frac{1}{2}|y^{(n)}|}\right)\varepsilon.
\end{equation*}
\end{proof}

\section{Derivation of Dual Dynamics for Gradient Descent}\label{appendixC}
When applying GD to minimize the loss (7) with respect to weights, in the continuous-time limit we have:
    \begin{align}
        \frac{d\bm w^+(t)}{dt} &= -\bm w^+(t) \odot X^\top (X \bm \beta(t) - \bm y) \label{eq:weightsuGD}\\
        \frac{d\bm w^-(t)}{dt} &= \bm w^-(t) \odot X^\top (X \bm \beta(t) - \bm y) \label{eq:weightsvGD}
    \end{align}
With initialization $\bm w^+(0) = \bm w^-(0) = \alpha \bm 1$, we can write implicit solutions to \eqref{eq:weightsuGD} and \eqref{eq:weightsvGD} as \begin{align*}
    \bm w^+(t) &= \alpha \exp\left(-\int_0^t X^\top(X\bm\beta(s) - \bm y)~ds\right), \\
    \bm w^-(t) &= \alpha \exp\left(\int_0^t X^\top(X\bm\beta(s) - \bm y)~ds  \right).
\end{align*}
Therefore, we have
\begin{align*}
    \bm\beta(t) &= \bm w^+(t) \odot \bm w^+(t)  - \bm w^-(t) \odot \bm w^-(t)
    \\
    &= \alpha^2 \left[\exp\left(-2\int_0^t X^\top(X\bm\beta(s) - \bm y)~ds\right) - \exp\left(2\int_0^t X^\top(X\bm\beta(s) - \bm y)~ds \right) \right] \\
    &= 2\alpha^2 \operatorname{sinh}\left(-2\int_0^t X^\top (X\bm\beta(s) - \bm y)~ds \right).
\end{align*}
It follows that
\begin{align*}
\frac{1}{2\alpha^2}\bm\beta(t) &=  \operatorname{sinh}\left(-2\int_0^t X^\top (X\bm\beta(s) - \bm y)~ds \right) \\
\operatorname{arcsinh}\left(\frac{\bm\beta(t)}{2\alpha^2}\right) &= -\int_0^t 2X^\top (X\bm\beta(s) - \bm y)~ds \\
\frac{d\operatorname{arcsinh}\left(\frac{\bm\beta(t)}{2\alpha^2}\right)}{dt} &= -2X^\top (X\bm\beta(t) - \bm y).
\end{align*}
We note that $\nabla_{\bm\beta} L(\bm\beta(t)) = \frac{1}{2}X^\top (X \bm\beta(t) - \bm y)$. Then, we have \begin{align}
\frac{d\operatorname{arcsinh}\left(\frac{\bm\beta(t)}{2\alpha^2}\right)}{dt} = -2X^\top (X\bm\beta(t) - \bm y) = -4\nabla_{\bm\beta} L(\bm\beta(t)).\label{eq:77}
\end{align}
Given the potential function $\Psi_{\alpha}(\bm\beta(t)) = \frac{1}{4}\left(\sum_{i=1}^D \beta_i\operatorname{arcsinh}(\frac{\beta_i}{2\alpha^2}) + \sqrt{\beta_i^2 + 4\alpha^4} \right)
$, we have the mirror map\begin{align}
    \nabla \Psi_\alpha(\bm\beta(t)) = \frac{1}{4}\operatorname{arcsinh}\left(\frac{\bm\beta(t)}{2\alpha^2}\right).\label{eq:78}
\end{align}
Using \eqref{eq:77} and \eqref{eq:78}, we get the dual dynamics for GD:\begin{align*}
    \frac{d\nabla \Psi_\alpha(\bm\beta(t))}{dt} = -\nabla_{\bm\beta}L(\bm\beta(t)).
\end{align*}

\end{document}